\documentclass[twoside,11pt]{article}
\usepackage{jair, theapa, rawfonts}

\usepackage[dvipsnames]{xcolor}

\usepackage{tabularx}
\usepackage{multirow}
\usepackage{booktabs}
\usepackage{placeins}

\usepackage{graphicx}
\usepackage{bm}

\usepackage{placeins}

\usepackage[utf8]{inputenc}
\usepackage{url} 
\usepackage{booktabs} 
\usepackage{nicefrac} 

\usepackage{amsmath}
\usepackage{amsthm}
\usepackage{amsfonts}
\usepackage{amssymb}
\usepackage{nccmath}

\usepackage[short]{optidef}

\usepackage[T1]{fontenc}

\newtheorem{thm}{Theorem}

\newtheorem{assume}[thm]{Assumption}

\DeclareMathOperator*{\loss}{loss}
\DeclareMathOperator*{\premium}{premium}

\DeclareMathOperator*{\nrmse}{nrmse}
\DeclareMathOperator*{\mean}{mean}

\newcommand{\Brack}[1]{\left( #1 \right)}

\begin{document}

\title{Fairness in Forecasting of Observations of \\Linear Dynamical Systems}

\author{\noindent 
       \name Quan Zhou \email q.zhou22@imperial.ac.uk \\
       \addr Dyson School of Design Engineering, \\Imperial College London, \\London, SW7 9EG, United Kingdom\\
       \addr School of Electrical and Electronic Engineering, \\University College Dublin, \\Dublin, D04 V1W8, Ireland
       \AND
       \name Jakub Mare\v{c}ek \email jakub.marecek@fel.cvut.cz \\
       \addr Department of Computer Science, \\Czech Technical University in Prague, \\Prague, 121 35, the Czech Republic
       \AND
       \name Robert Shorten \email r.shorten@imperial.ac.uk\\
       \addr Dyson School of Design Engineering, \\Imperial College London, \\London, SW7 9EG, United Kingdom\\
       \addr School of Electrical and Electronic Engineering, \\University College Dublin,\\ Dublin, D04 V1W8, Ireland}

\maketitle

\begin{abstract}
In machine learning, training data often capture the behaviour of multiple subgroups of some underlying human population.
This behaviour can often be modelled as observations of an unknown dynamical system with an unobserved state.
When the training data for the subgroups are not controlled carefully,  however, under-representation bias arises.
To counter under-representation bias, we introduce two natural notions of fairness in time-series forecasting problems: subgroup fairness and instantaneous fairness.
These notions extend predictive parity to the learning of dynamical systems.
We also show globally convergent methods for the fairness-constrained learning problems using hierarchies of convexifications of non-commutative polynomial optimisation problems. 
We also show that by exploiting sparsity in the convexifications, we can reduce the run time of our methods considerably.
Our empirical results on a biased data set motivated by insurance applications and the well-known COMPAS data set demonstrate the efficacy of our methods.
\end{abstract}
\section{Introduction}\label{sec1}

Forecasts affect almost all aspects of our daily life, as a basis for access-control mechanisms. 
As the quality of the forecasts impacts our lives, for better or worse, it is becoming more and more apparent that many of the tools that produce the forecasts seem to be, or indeed are, unfair, in a sense we formalise below.
If the tools used to produce the forecasts are unfair, society suffers.

One such example of a forecasting tool that shapes the very pillars of our society is the \emph{FICO Score} \cite{FICO}. FICO is a measure of an individual's creditworthiness (or conversely, credit-default risk), computed by the Fair Isaac Corporation.
It has been suggested \cite{nickerson2016asset,delis2019mortgage,Hegarty9152,vogel2021learning} that the FICO Score may be unfair to certain minorities, although this has been disputed \cite{avery2012does}.

As another example, consider college admissions, where results of standardised tests were often presented as forecasts of potential academic success.
In the context of COVID-19, forecasting algorithms utilising previous grades and input from teachers replaced standardised tests in determining the satisfaction of college-admission requirements in many jurisdictions. 

Other forecasting tools are perhaps less well known, but perhaps even more alarming, as they have the potential to shape the very core of society. One striking example is Northpointe's {\em Correctional Offender Management Profiling for Alternative Sanctions} (COMPAS). COMPAS is a criminal-risk assessment tool that is widely used in pretrial, parole, and sentencing decisions at courts in New York, Wisconsin, California, and Florida.
COMPAS forecasts the likelihood that an individual will re-offend within two years. 
It has been suggested \cite{angwin2016machine,dressel2021dangers} that COMPAS under-predicts recidivism for Caucasian defendants, and over-predicts recidivism for African-American defendants.\footnote{
We note that this has been disputed \cite{kleinberg2017inherent,dressel2021dangers}, and that it has been suggested that such forecasts \cite{Neile2107020118} are difficult to make, in general, due to the cohort differences in group-based arrest trajectories.}

The applications, where fairness seems most important, often capture the behaviour of multiple subgroups of some underlying human population in the training data.
Let us consider a model, where there are a number of individuals within a population $\mathcal{P}$. 
The population $\mathcal{P}$ is partitioned into subgroups indexed by $\mathcal{S}$. 
For each subgroup $s\in\mathcal{S}$, there is a set $\mathcal{I}^{(s)}$ of trajectories of observations available and each trajectory $i \in \mathcal{I}^{(s)}$ has observations for periods $\mathcal{T}^{(i,s)}$, possibly of varying cardinality $\lvert\mathcal{T}^{(i,s)}\rvert$.
Each subgroup $s \in\mathcal{S}$ is associated with a model, $\mathcal{L}^{(s)}$.
For all $i \in \mathcal{I}^{(s)}$, $s\in\mathcal{S}$, the trajectory $\{Y_t\}^{(i,s)}$, for $t\in \mathcal{T}^{(i,s)}$, is hence generated by precisely one model $\mathcal{L}^{(s)}$. 
Throughout, the superscripts distinguish the trajectories and subgroups, while subscripts indicate the periods.    
    

    

In this setting, under-representation bias \cite[cf. Section 2.2]{blum2019recovering} arises, where the trajectories of observations from one (``disadvantaged'') subgroup are under-represented in the training data.
This is particularly important if the forecasting is constrained to be subgroup-blind, i.e., 
we wish to learn a single subgroup-blind model $\mathcal{L}$.
This is the case when the use of protected attributes distinguishing each subgroup can be 
regarded as discriminatory, such as in the case of gender and race \cite{gajane2017formalizing,kleinberg2018discrimination}. 
Notice that such anti-discrimination measures are increasingly stipulated legally, e.g., within insurance pricing, where the sex of the applicant cannot be used, despite being known.
More broadly, under-representation bias harms both the accuracy of the forecast and fairness in the sense of varying accuracy across the subgroups.

%



To address under-representation bias in the training of a forecasting model, it is natural to seek a notion of fairness 
that captures the overall behaviour across all subgroups, 
while taking into account the varying amounts of training data for the individual subgroups. 
To formalise this, suppose that we learn one model $\mathcal{L}$ from the multiple trajectories and define a loss function that measures the loss of accuracy for a certain observation $Y_t^{(i,s)}$ when adopting the forecast $f_t$ for the overall population. For $t\in \mathcal{T}^{(i,s)}$, $i \in \mathcal{I}^{(s)}$, $s\in\mathcal{S}$, we have
\begin{equation}
    \loss^{(i,s)}(f_t):=||Y_t^{(i,s)}-f_t||. \label{equ:loss}
\end{equation}
Let $\mathcal{T}^+=\cup_{i\in\mathcal{I}^{(s)},s\in\mathcal{S}}\mathcal{T}^{(i,s)}$. Following the definitions of $\mathcal{T}^{(i,s)}$, there is not an observation for $t\notin \mathcal{T}^+$, in which case, loss in Equation~\eqref{equ:loss} does not exist.
To evaluate the performance of the forecasts, we only consider $f_t$ made in periods $t\in\mathcal{T}^+$. 
Note that, since each trajectory is of varying length, it is possible that for a certain triple $(t,i,s)$, there is no observation $Y_t^{(i,s)}$.


Following much recent work on fairness in classification, e.g., \cite{zliobaite2015relation,hardt2016equality,kilbertus2017avoiding,kusner2017counterfactual,chouldechova2020snapshot,aghaei2019learning},
we propose two  objectives to address the under-representation bias, which extend group fairness \cite{feldman2015certifying} to time series:

\begin{enumerate}
    \item \textbf{Subgroup Fairness}. 
    The objective seeks to equalise, across all  
    subgroups, the sum of losses for the subgroup.
    Considering the number of trajectories in each subgroup and the 
    number of observations across the trajectories may differ, 
    we include $\lvert\mathcal{I}^{(s)}\rvert,\lvert\mathcal{T}^{(i,s)}\rvert$ as weights:
    \begin{equation}
        \min_{f_t,t\in\mathcal{T}^+} \max_{s\in\mathcal{S}} \left \{
        \frac{1}{\lvert\mathcal{I}^{(s)}\rvert}
        \sum_{i \in \mathcal{I}^{(s)}} \frac{1}{\lvert\mathcal{T}^{(i,s)}\rvert}
        \sum_{t\in \mathcal{T}^{(i,s)}} \loss^{(i,s)}(f_t) \right \} \label{equ:obj-Subgroup-Fair}
    \end{equation}
    \item \textbf{Instantaneous Fairness}. The objective seeks to equalise 
    the instantaneous loss, by minimising the 
    maximum of the losses across all subgroups and all times: 
    \begin{equation}
        \min_{f_t,t\in\mathcal{T}^+}
        \left \{ 
        \max_{t\in\mathcal{T}^{(i,s)},i\in\mathcal{I}^{(s)},s\in\mathcal{S}} \left \{ \loss^{(i,s)}(f_t) \right \} \right \} \label{equ:obj-Instant-Fair}
    \end{equation}

\end{enumerate}



\paragraph{Specific Contributions: }This paper builds on our preliminary work that was presented in AAAI 2021 \cite{zhou2021fairness}. The present manuscript extends this work by:
\begin{itemize}
    \item We have extended our fairness notions presented at \cite{zhou2021fairness} to more generally applicable post-processing methods.
    \item We have added a comparison of our fairness notions with all post-processing methods in the AI Fairness 360 library, which are based on previous fairness notions of ``calibrated equalised odds'', ``equalised odds'' and ``demographic parity'', respectively.
\end{itemize}
With respect to the state of the art, this paper defines two new notations of fairness.
We then cast the learning of a linear dynamical system with such fairness considerations as a non-commutative polynomial optimisation problem (NCPOP), which can be solved efficiently using a globally-convergent hierarchy of semidefinite programming (SDP) relaxations, which can be of independent interest.
A comprehensive comparison is given to illustrate the efficacy of our approach.

\section{Definitions and Related Work}


Recent years have seen an unprecedented explosion in attention of notions of fairness in the field of artificial intelligence and machine learning \cite{ntoutsi2020bias,chouldechova2020snapshot}. 
In a typical machine learning process, the training set usually contains individuals' protected attributes (e.g., race, gender), remaining attributes, and a target variable $Y$.
Other than not using the protected attributes (``fairness under unawareness''), several candidate definitions of fairness have been proposed and we would start from statistical notions of fairness, which could be roughly categorised into three types \cite{barocas2017fairness}: (i) independence; (ii) separation; and (iii) sufficiency; a more fine-grained discussion of fairness is presented in the sequel.
The independence notion typically asks the output to be independent of protected attributes. A simple example is the ``demographic parity'' \cite{calder2009experimental}, which requires each segment of a protected class (e.g., defined by gender) to receive the positive outcome at equal rates.
The notion of separation, e.g., ``equal odds'' or ``equal opportunity'' in \cite{hardt2016equality}, requires the predictor's output to be unrelated to protected attributes, but conditional on the target variables $Y$.
Finally, the notion of sufficiency, derived from calibration, asked the target variables be independent from protected attributes conditional on the predictor output. 
For instance, we would expect the portion of defendants who were predicted to re-offend by the COMPAS system and actually re-offend to be equalised across subgroups.
Calibration would require that for any given COMPAS score, the recidivism rates are similar.

The notions of independence, separation, and sufficiency are all related to subgroups of the population and provide an average guarantee for individuals in the protected group \cite{awasthi2020beyond}. 
In contrast, the notion of individual fairness asks for constraints that bind on specific pairs of individuals, rather than on a quantity that is averaged over groups \cite{chouldechova2020snapshot}. In other words, it requires ``similar individuals should be treated similarly'' \cite{petersen2021post,dwork2012fairness}.
However, this notion requires a similarity metric capturing the ground truth, which requires general and task-specific assumption on its definition \cite{sharifi2019average}.
Apart from notions based on correlations of statistical measures and fairness, the notion of counterfactual fairness, pioneered by \cite{kusner2017counterfactual}, operates at the individual level such that causal methods are used to examine whether a decision is the same as in situations whether an individual's protected attributes are altered or not.
Its generalised variant path-specific fairness in \cite{nilforoshan2022causal,chiappa2019path} specifies the effects of protected attributes along certain path in a causal directed acyclic graph.
The notion of ``procedural fairness'', or in other words, the fairness of the decision-making process,
especially in processes that resolve disputes and allocate resources,
deeply rooted in legal science, has recently been tied to causal methods and fair feature selection in \cite{belitz2021automating,grgic2018beyond}.
As the many notions of fairness arise, it is necessary to build up a comprehensive framework of multiple fairness criteria, especially when there is not a widely-recognised trivial fairness notion should be used \cite{awasthi2020beyond} or when certain notions are incompatible with one another \cite{chakraborty2022weighted}. One could consider dynamically learning fair policies using feedback \cite{wen2021algorithms,d2020fairness}.


Somewhat removed from the mainstream literature, there are some excellent works in \cite{foulds2020intersectional,yang2020fairness}, who give the definition of independent subgroups, intersectional subgroups, and gerrymandering subgroups for the situations of overlapping subgroups. 
\cite{tran2021differentially,pinzon2021impossibility,chang2021privacy,cummings2019compatibility} discuss the trade-off between differential privacy \cite{Dwork2006differential} which is an important direction for further work.
An empirical study of real-world problems in \cite{rodolfa2021empirical} challenged the existence or magnitude of the trade-off accuracy between fairness.
While maintaining accuracy, there exist tension between fairness and interpretability (feature deduction) \cite{agarwal2021trade}. 
For the cases of unknown or ambiguous protected attributes, \cite{amini2019uncovering} introduce the latent variables which may over-represent some subgroups as the proxies of unknown protected attributes.
\cite{wang2020robust} use the technologies of distributional robust optimisation to minimise the worse-case expected loss of the predictor with an upper bound on the distance between the distribution of ambiguous and actual protected attributes.
\cite{zhang2021assessing,Jeong2021FairnessWI} discuss the impact of missing data on fairness in uniformly sampled time series, which our model of non-uniformly sampled trajectories largely avoids. Some literature in \cite{caton2022impact,fernando2021missing} also discusses imputation strategies for fairness with respect to missing data.



Our approach to addressing the under-representation bias is rooted within the imbalanced-learning literature e.g., \cite{he2013imbalanced,rolf2021representation} and presents a step forward within the fairness in forecasting studied recently by \cite{gajane2017formalizing,chouldechova2017fair,NIPS20199603,Jeong2021FairnessWI}, as outlined in the excellent survey of   
\cite{chouldechova2020snapshot,barocas2017fairness}. 
On a more technical level, our work on fairness in learning linear dynamical systems is complemented by several recent studies involving dynamics and fairness \cite{mouzannar2019fair,paassen2019dynamic,jung2020fair}, and several even more recent studies on the learning of non-linear dynamics \cite{salmela2021predicting,lu2021learning}.
We rely crucially on tools developed in non-commutative polynomial optimisation \cite{Pironio2010,wang2019tssos,wang2020chordal} and non-commutative algebra \cite{gelfand1943imbedding,segal1947irreducible,mccullough2001factorization,helton2002positive}, which have not seen much use in Statistics and Machine Learning, yet.  

\section{Our Models}
\label{sec:models}

As the simplest example of the use of subgroup fairness and instantaneous fairness, cf. Equations~\eqref{equ:obj-Subgroup-Fair} and \eqref{equ:obj-Instant-Fair}, consider their applications in linear regression. For simplicity, let us assume that the cardinality of each subgroup is the same and the lengths of all trajectories are equal. Then:
\begin{equation*}
\begin{split}
    &\min_{z,A,f_t,t\in\mathcal{T}^+} z\qquad\qquad\qquad\qquad\quad\textrm{Subgroup Fairness in LR} \\
     &\textrm{s.t.} \quad z \geq \sum_{i\in\mathcal{I}^{(s)} ,t\in\mathcal{T}^+}\lvert Y^{(i,s)}_t - f_t\rvert,\;\forall s\in\mathcal{S}\\ 
     &\qquad f_t = A X_t,  \\
    \vspace{10pt}
    &\min_{z,A,f_t,t\in\mathcal{T}^+} z \qquad\qquad\qquad\quad\textrm{Instantaneous Fairness in LR}\\ 
    &\textrm{s.t.} \quad z \geq \lvert Y^{(i,s)}_t - f_t \rvert,\;\forall i\in\mathcal{I}^{(s)},s\in\mathcal{S},t\in\mathcal{T}^+\\
    &\qquad f_t = A X_t,     
\end{split}
\end{equation*}
where $A$ concatenates the regression coefficients. $X_t$ concatenates explanatory variables. $f_t$ is the dependent variable and $Y_t^{(i,s)}$ is the actual observation in a compatible fashion. 
The auxiliary scalar variable $z$ is used to reformulate ``$\max$'' in the objective in Equations~\eqref{equ:obj-Subgroup-Fair} and \eqref{equ:obj-Instant-Fair}.

Next, let us consider more elaborate models, which assume that there exists a linear dynamical system (LDS) 
corresponding to each subgroup $s\in\mathcal{S}$.
A discrete-time model of a linear dynamical system $\mathcal{L}=(G,F,V,W)$ \cite{WestHarrison} suggests that the random variable $Y_t \in\mathbb{R}^{m}$ capturing the observed component ( i.e., output, observations or measurements) evolves over time $t\geq 1$ according to:
\begin{equation}
\begin{split}
    \phi_{t}  &= G \phi_{t-1} + w_t, \\
    Y_t &= F' \phi_t + v_t, 
\end{split}
\label{equ:LDS}
\end{equation}
where $\phi_t \in \mathbb{R}^{n}$ is the  hidden component (state) and $G \in \mathbb{R}^{n\times n}$ and $F\in\mathbb{R}^{n\times m}$ are compatible system matrices. Random variables $w_t,v_t$ capture normally-distributed process noise and observation noise, with zero means and covariance matrices $W \in\mathbb{R}^{n\times n}$ and $V \in\mathbb{R}^{m\times m}$, respectively.



The objectives in Equations~\eqref{equ:obj-Subgroup-Fair} and \eqref{equ:obj-Instant-Fair}, subject to the state-evolution and observation equations, in Equation~\eqref{equ:LDS}, yield two operator-valued optimisation problems. 
Their inputs are $Y_{t}^{(i,s)},t\in\mathcal{T}^{(i,s)},i\in\mathcal{I}^{(s)},s\in\mathcal{S}$, i.e., the observations of multiple trajectories and the multipliers $\lambda_1,\lambda_2>0$. The operator-valued decision variables $\mathcal{O}$ include operators $F, G$, vectors $m_t, \omega_t$, and scalars $f_t, \nu_t, z$. Notice that $t$ ranges over $t\in\mathcal{T}^+$, except for $m_t$, where $t\in\mathcal{T}^+\cup\{0\}$. The auxiliary scalar variable $z$ is used to reformulate ``$\max$'' in the objective in Equations~\eqref{equ:obj-Subgroup-Fair} and \eqref{equ:obj-Instant-Fair}.
Since the process noise and observation noise are assumed to be samples of mean-zero normally-distributed random variables, we add the sum of squares of $\omega_t$ (resp. $\nu_t$) to the objective with the positive multiplier $\lambda_1$ (resp. $\lambda_2$), seeking a solution with $\omega_t$ (resp. $\nu_t$) close to zero.
Overall, the subgroup-fair and instant-fair formulations read:
\begin{mini} 
	  {\mathcal{O}}{z + \lambda_1 \sum_{t\geq 1} \omega_t^2+\lambda_2 \sum_{t\geq 1} \nu_t^2}{\label{min:FairA}}{ 
	  }{\textrm{Subgroup-Fair}}
	  \addConstraint{z}{\geq
	  \frac{1}{\lvert\mathcal{I}^{(s)}\rvert}
        \sum_{i \in \mathcal{I}^{(s)}} \frac{1}{\lvert\mathcal{T}^{(i,s)}\rvert}
        \sum_{t\in \mathcal{T}^{(i,s)}} \loss^{(i,s)}(f_t)}
        {, s\in\mathcal{S}}
	  \addConstraint{m_t}{= G m_{t-1}+\omega_t}{, t\in\mathcal{T}^+}
	  \addConstraint{f_t}{= F' m_{t}+\nu_t}{, t\in\mathcal{T}^+.}
\end{mini}
\begin{mini} 
	  {\mathcal{O}}{z  + \lambda_1 \sum_{t\geq 1} \omega_t^2+\lambda_2 \sum_{t\geq 1} \nu_t^2}{\label{min:FairB}}{
	  }{\textrm{Instant-Fair}}
	  \addConstraint{z}{\geq
	  \loss^{(i,s)}(f_t)}{, t\in\mathcal{T}^{(i,s)},i\in\mathcal{I}^{(s)},s\in\mathcal{S}}
	  \addConstraint{m_t}{= G m_{t-1}+\omega_t}{, t\in\mathcal{T}^+}
	  \addConstraint{f_t}{= F' m_{t}+\nu_t}{, t\in\mathcal{T}^+.}
\end{mini}

For comparison, we use a traditional formulation that focuses on minimising the overall loss: 
\begin{mini}
	  {\mathcal{O}}{
	  \sum_{s \in \mathcal{S}} \quad
	  \sum_{i \in \mathcal{I}^{(s)}} 
	  \sum_{t\in \mathcal{T}^{(i,s)}} \loss^{(i,s)}(f_t) + \lambda_1 \sum_{t\geq 1} \omega_t^2+\lambda_2 \sum_{t\geq 1} \nu_t^2}{\label{min:Unfair}}{ 	  }{\quad\textrm{Unfair}} 
	  \addConstraint{\qquad\qquad m_t}{= G m_{t-1}+\omega_t}{, t\in\mathcal{T}^+}
	  \addConstraint{\qquad\qquad f_t}{= F' m_{t}+\nu_t}{, t\in\mathcal{T}^+.}
\end{mini}


As we explain in the Appendix, the operator-valued optimisation problems (i.e., ``Unfair'', ``Instant-Fair'', and ``Subgroup-Fair'') can be convexified to any given accuracy, and thence solved efficiently, under a technical assumption
related to the stability of the LDS, 
which entails that the estimates of states and observations remain bounded, and thus all operator-valued decision variables remain bounded.

\section{Numerical Illustrations}

Under-representation bias considers the situation where some subgroups would be given unfair treatments, either due to the varying numbers or lengths of trajectories across subgroups.
In response to under-representation bias, we have introduced two natural fairness notions for forecasting.
We use a Linear Dynamic System to predict the next observation, as in Equation~\eqref{equ:LDS}.
Then, we have given two formulations associated with each notion. We tested our formulations on the famous Correctional Offender Management Profiling for Alternative Sanctions (COMPAS) dataset. 
Our implementation is available on-line at \url{https://github.com/Quan-Zhou/Fairness-in-Learning-of-LDS}.



\subsection{Generation of Biased Training Data}
\label{sec:Biased Training Data Generalisation}

To illustrate the impact of our models on data with varying degrees of under-representation bias, we consider a method to generate data with a given degree of bias, which is based on \cite[cf. Section 2.2]{blum2019recovering}.
Suppose that there is one advantaged subgroup (a) and one disadvantaged subgroup (d), i.e., $S=\{$a, d$\}$, with trajectories $\mathcal{I}^{(a)}$ and $\mathcal{I}^{(d)}$ in each subgroup. Under-representation bias can enter the training set in the following steps:
\begin{enumerate}
\item Consider that the LDS for both subgroups $\mathcal{L}^{(s)},s\in\mathcal{S}$ have the same system matrices: $$G^{(s)}=\begin{bmatrix}
0.99 & 0 \\
1.0 & 0.2 \end{bmatrix},F^{(s)}=\begin{bmatrix}
1.1 \\ 0.8 
\end{bmatrix},$$
while the covariance matrices $V^{(s)},W^{(s)},s\in\mathcal{S}$ are sampled randomly from a uniform distribution over $[0,1)$ and $[0,0.1)$, respectively. The initial states $m_0^{(s)}$ of both subgroups are $5$ and $7$.
    \item Observations $Y^{(i,s)}_t$ are sampled from corresponding LDS $\mathcal{L}^{(s)}$. Thus each $Y^{(i,s)}_t\sim\mathcal{L}^{(s)}$.
    \item Let $\mathcal{\beta}^{(d)}$ denote the probability that an observation from subgroup $d$ stays in the training data, and $0\leq \mathcal{\beta}^{(d)} \leq 1$. It can be seen as the ratio of the number of observations in disadvantaged subgroup to that of advantaged subgroup. The degree of under-representation bias can be controlled by simply adjusting $\mathcal{\beta}^{(d)}$.
    Smaller values of $\mathcal{\beta}^{(d)}$ correspond to higher level of bias in the training set.
\end{enumerate}
The last step makes the number of observations of the disadvantaged subgroup less than that of the advantaged subgroup when $0\leq \mathcal{\beta}^{(d)} <1$.
Hence, the advantaged subgroup becomes over-represented. 
Note that for a small sample size, it is necessary to make sure that there is at least one observation in each subgroup at each period. 

\subsection{Effects of Under-Representation Bias on Forecast}
\begin{figure}[!t]
    \centering
\includegraphics[width=0.6\textwidth]{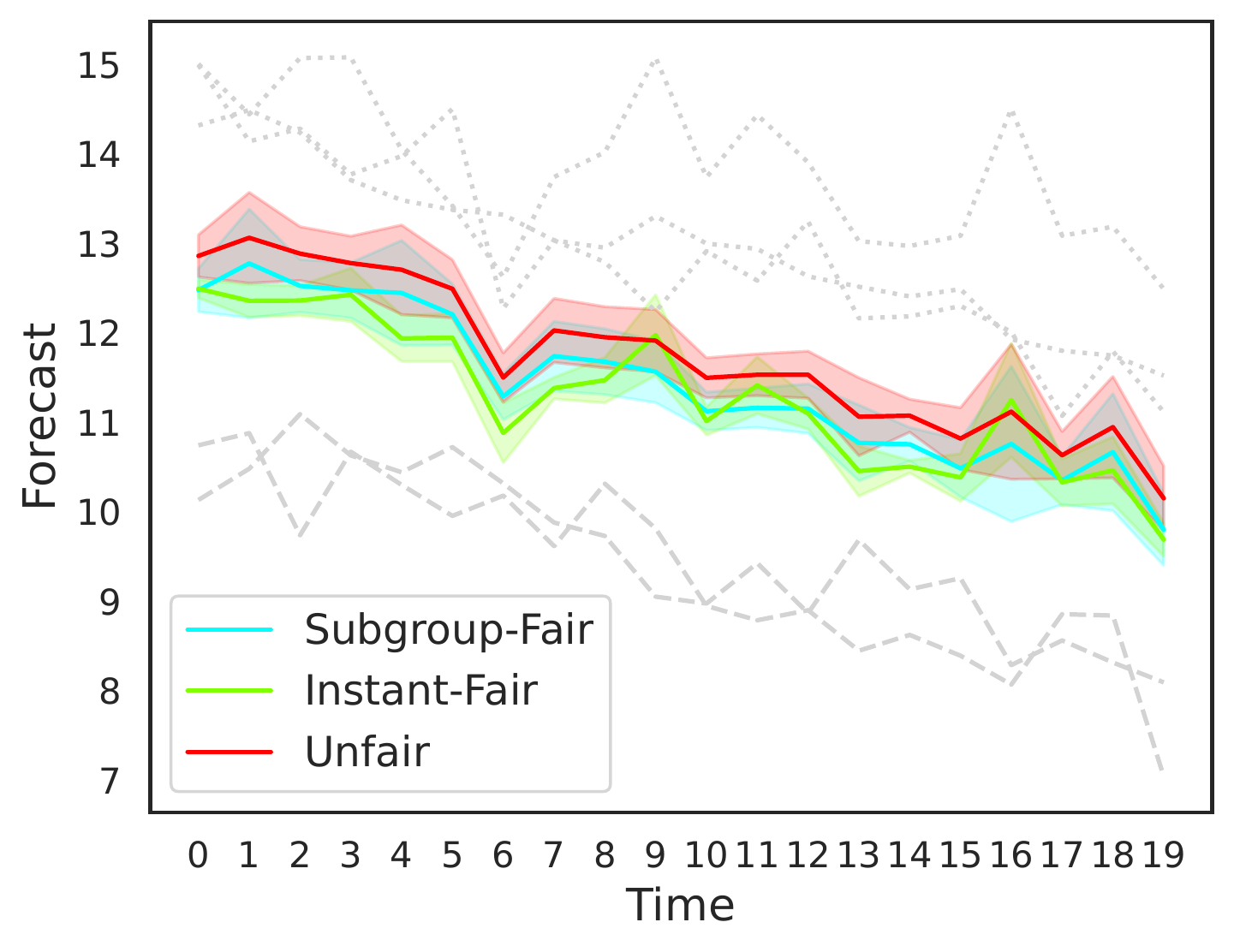} 
\caption{Forecast obtained using (\ref{min:FairA}-\ref{min:Unfair}): the solid lines in primary colours with error bands display the mean and standard deviation of the forecasts over 10 experiments. For reference, dotted lines and dashed lines in grey denote the trajectories of observations of advantaged and disadvantaged subgroups, respectively, before discarding any observations.
}
\label{fig:LinePlot}
\end{figure}

To illustrate the impact of our models on data with varying degrees of under-representation bias, suppose there is an advantaged subgroup (a) and a disadvantaged subgroup (d), i.e., $\mathcal{S}=\{a,d\}$.
Figure~\ref{fig:LinePlot} illustrates 10 experiments with general forecasting procedures. For each experiment, the same set of observations $Y_t^{(i,s)},t\in \mathcal{T}^{(i,s)}$, $i \in \mathcal{I}^{(s)}$, $s\in\mathcal{S}$ is reused, and the trajectories of advantaged and disadvantaged subgroups are denoted by dotted curves and dashed curves, respectively. However, in each experiment, a subset of observations with the same cardinality is randomly selected and discarded and thus a new biased training set is generated, albeit based on the same ``ground set'' of observations. The three models in Equations~\eqref{min:FairA}-\eqref{min:Unfair} are applied in each experiment with $\lambda_1$ of 1, 3, and 5, respectively, as chosen by iterating over integers 1 to 10, while $\lambda_2$ remains 0.01, 
The mean of forecast $f_t$ across 10 experiments and its standard deviation are shown as solid curves with error bands. The red curve gives an overview of how a prediction without considering fairness would cause an unevenly distributed prediction loss for each subgroup. This is simply because the advantaged subgroup is of larger cardinality, and the overall loss would decrease more steeply if the predicted trajectory gets closer to the advantaged subgroup.



\subsection{Fairness as a Function of Bias}
\begin{figure}[!t]    
    \centering
\includegraphics[width=0.6\textwidth]{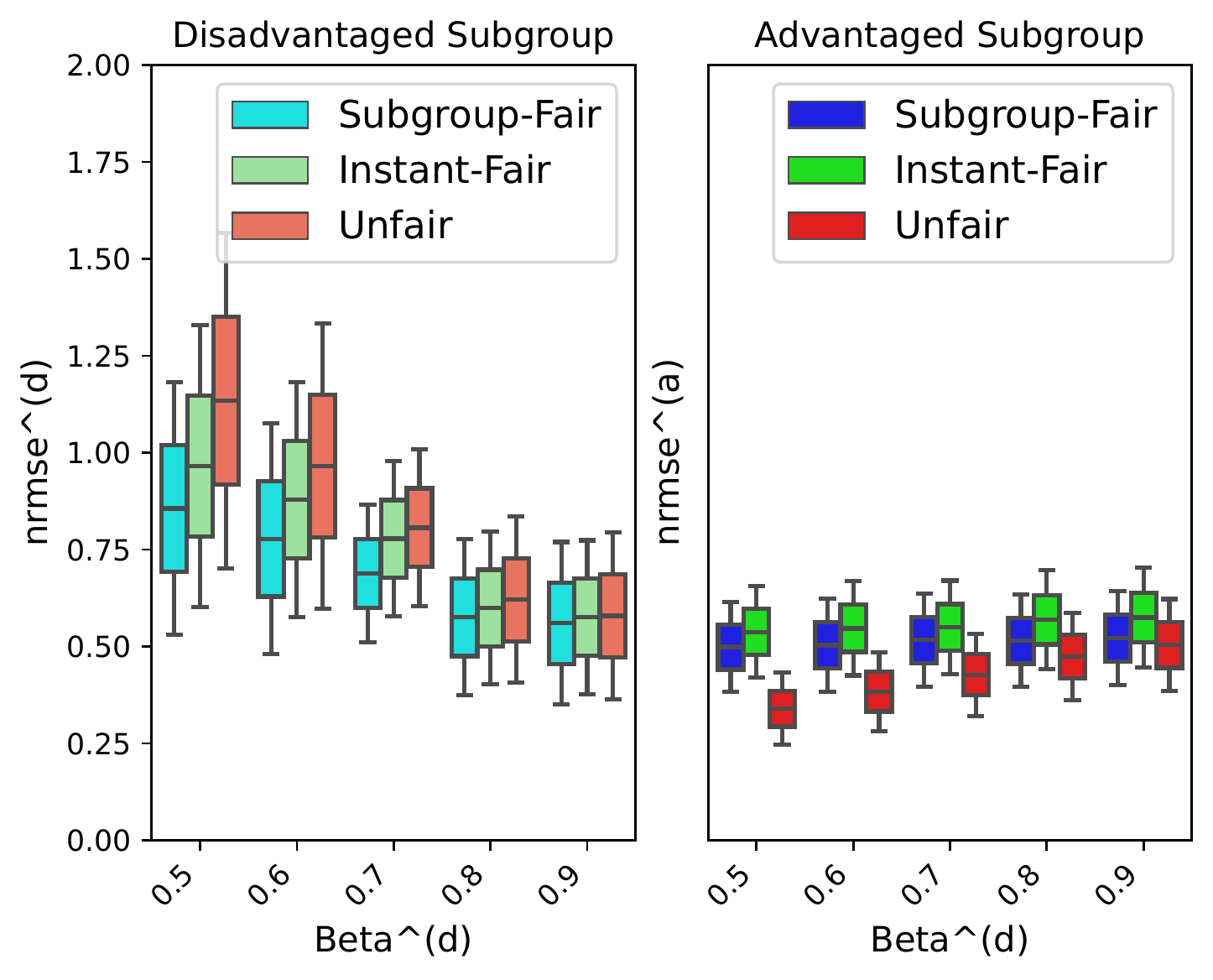}
\caption{Accuracy as a function of the degree of under-representation bias: the boxplot of $\nrmse^{(s)},s\in\mathcal{S}$ against $\mathcal{\beta}^{(d)}$, where $\mathcal{\beta}^{(d)}=[0.5,0.55,\dots,0.9]$. Each box for the quartiles of $\nrmse^{(s)}$ is obtained from $10$ experiments, with observations generated in the same way as the ones in Figure~\ref{fig:LinePlot}.}
\label{fig:BetaPlot}
\end{figure}

Figure~\ref{fig:BetaPlot} suggests how the degree of bias affects accuracy in each subgroup with and without considering fairness. 
With the number of trajectories in both subgroups set to two, i.e., $|I_a|=|I_d|=2$, we vary the degree of bias by adjusting $\mathcal{\beta}^{(d)}$ within the range of $[0.5, 0.9]$. 
To measure the effect of the degree on accuracy, we introduce the normalised root mean square error (nrmse) fitness value for each subgroup $s\in\mathcal{S}$: 
\begin{equation*}
    \mathrm{\nrmse^{(s)}}:=\sqrt{
    \frac{\sum_{i\in\mathcal{I}^{(s)}}\sum_{t\in\mathcal{T}^{(i,s)}}\Brack{ Y_t^{(i,s)}-f_t}^2}{\sum_{i\in\mathcal{I}^{(s)}}\sum_{t\in\mathcal{T}^{(i,s)}}\Brack{ Y_t^{(i,s)}-\mean^{(s)}}^2} },
\end{equation*}
where $\mean^{(s)}:=\frac{1}{\lvert\mathcal{I}^{(s)}\rvert} \sum_{i\in\mathcal{I}^{(s)}}\frac{1}{\lvert\mathcal{T}^{(i,s)}\rvert}\sum_{t\in\mathcal{T}^{(i,s)}} Y_t^{(i,s)}$. Higher $\nrmse^{(s)}$ indicates lower accuracy for subgroup $s$, i.e., the predicted trajectory of subgroup-blind $\mathcal{L}$ is further away from this subgroup. 

The training data are generated in the same way as the set of observations used in Figure~\ref{fig:LinePlot}, but with two trajectories in each subgroup ($|I_a|=|I_d|=2$).
Then, the biased training data generalisation process (described in Section~\ref{sec:Biased Training Data Generalisation}) is applied in each experiment with the value of $\mathcal{\beta}^{(d)}$ selecting from $0.5$ to $0.9$ at the step of $0.1$.
For each value of $\mathcal{\beta}^{(d)}$, three models in Equations~\eqref{min:FairA}-\eqref{min:Unfair} are conducted for $10$ experiments with a new biased training set in each experiment.
Therefore, the quartiles of $\nrmse^{(s)}$ across $10$ experiments for each subgroup are shown as boxes in Figure~\ref{fig:BetaPlot}.

One could expect that nrmse fitness values of the advantaged subgroup in Figure~\ref{fig:BetaPlot} to be generally lower than those of the disadvantaged subgroup ($\nrmse^{(d)}\geq\nrmse^{(a)}$), leaving a gap. Those gaps narrow down as $\mathcal{\beta}^{(d)}$ increases, simply because more observations of disadvantaged subgroup remain in the training data. Compared the to ``Unfair'', models with fairness constraints, i.e., ``Subgroup-Fair'' and ``Instant-Fair'', show narrower gaps and higher fairness between two subgroups. More surprisingly, when $\nrmse^{(a)}$ decreases as $\mathcal{\beta}^{(d)}$ gets close to $0.5$, ``Subgroup-Fair'' model still can keep $\nrmse^{(d)}$ at almost the same level, indicating a rise in overall accuracy. This is in contrast to the results of \cite{zliobaite2015relation,dutta2019information} in classification, but in line with recent work \cite{maity2021does}.

\subsection{Evaluation of Run Time of the Method}

Minimising multivariate operator-valued polynomial optimisation problems (\ref{min:FairA}-\ref{min:Unfair}) is a known non-trivial problem. We  exploit
sparsity-exploiting variants (TSSOS) of the globally convergent Navascués-Pironio-Acín (NPA) hierarchy used in the proof of Theorem \ref{T:covergence}, to develop fast computational methods.
See \cite{klep2019sparse,wang2019tssos,wang2020chordal,wang2020exploiting}.
The SDP of a given order in the respective hierarchy can be constructed using \texttt{ncpol2sdpa} 1.12.2\footnote{\url{https://github.com/peterwittek/ncpol2sdpa}} of \cite{wittek2015algorithm} or the tools of \cite{wang2020exploiting} \footnote{\url{https://github.com/wangjie212/TSSOS}} and then solved by \texttt{mosek} 9.2 of \cite{mosek2020mosek}.


In Figure~\ref{fig:TimePlot}, we illustrate the run-time and size of the relaxations as a function of the length of the time window.
Models ``Subgroup-Fair'' in Equation~\eqref{min:FairA} and ``Instant-Fair'' in Equation~\eqref{min:FairB} are implemented three times for each length of the time window, with the same data set used in Figure~\ref{fig:BetaPlot}.
The type of models, i.e., ``Subgroup-Fair'' (solid curves) and ``Instant-Fair'' (dashed curves), is distinguished by line styles.
The deep-pink and cornflower-blue curves show the run-time of the first-order SDP relaxation of NPA and the second-order SDP relaxation of TSSOS hierarchy, respectively, implemented with five CPUs and 64GB of memory per CPU.
The mean and mean $\pm$ 1 standard deviation of run-time across 3 experimental runs are presented by curves with shaded error bands.
The grey curve displays the number of variables in the first-order SDP relaxation of our models in Equations~\eqref{min:FairA} against the length of time window.
Further, models ``Subgroup-Fair'' in Equation~\eqref{min:FairA} and ``Instant-Fair'' in Equation~\eqref{min:FairB} are implemented once via TSSOS for each length of the time window, using COMPAS dataset, as the experiment in Figure~\ref{fig:COMPASPlot}, with the run-time displayed by a coral solid curve and a coral dashed curve, respectively.
It is clear that the run-time of TSSOS exhibits a modest growth with the length of time window, while that of the plain-vanilla NPA hierarchy grows much faster.

\begin{figure}
\centering{
\includegraphics[width=0.65\textwidth]{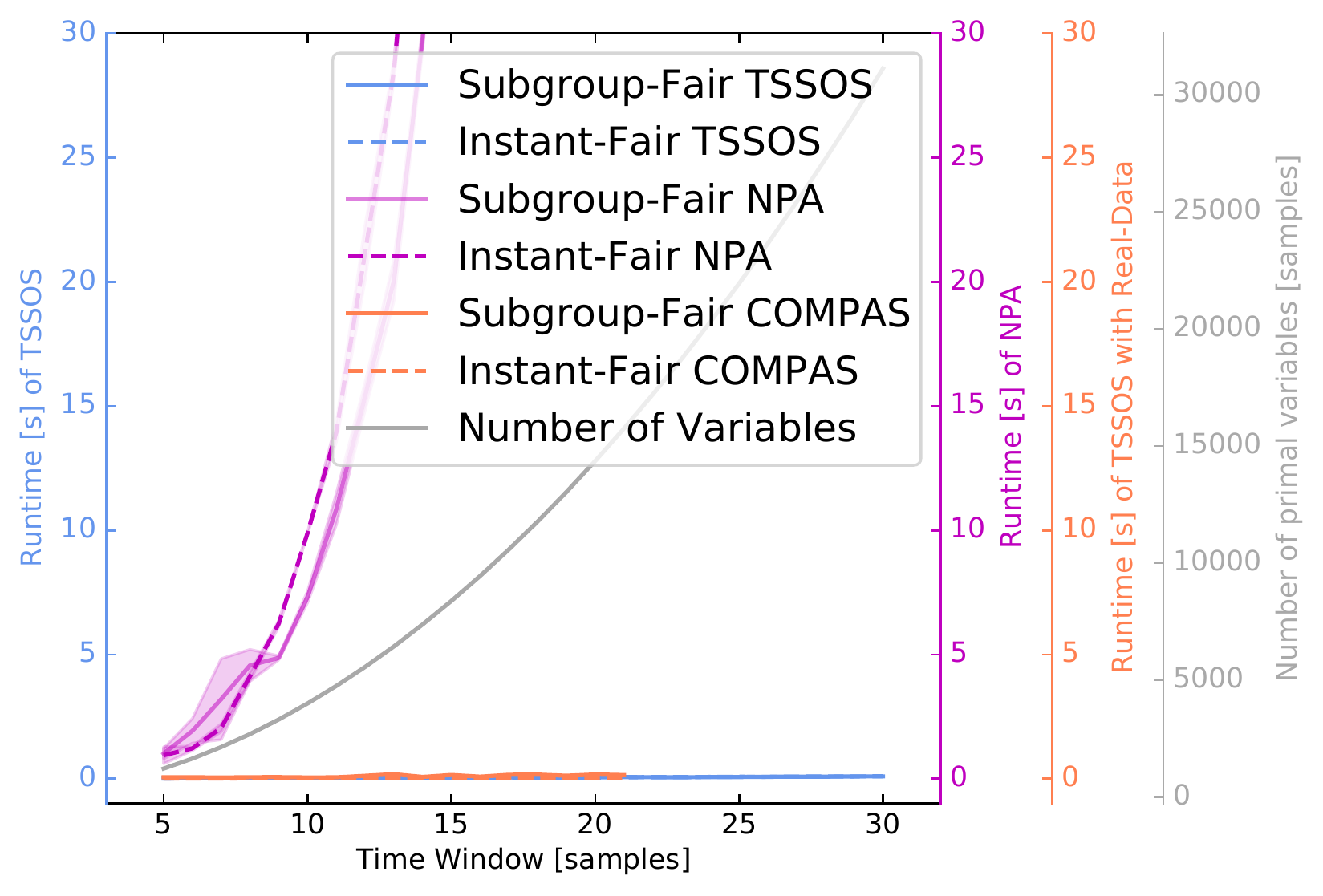} }
\caption{
The dimensions of relaxations and the run-time of SDPA thereupon as a function of the length of time window. Run-time of TSSOS and NPA is displayed in cornflower-blue and deep-pink curves, respectively, while the grey curve shows the number of variables in relaxations. Additionally, the run-time of TSSOS using the COMPAS dataset in Figure~\ref{fig:COMPASPlot}, is also displayed as coral-coloured curves. 
For run-time, the mean and mean $\pm$ one standard deviations across three experimental runs are presented by curves with shaded error bands.
}
\label{fig:TimePlot}
\end{figure}

\section{Numerical Results of COMPAS Dataset}
Finally, we wish to suggest the broader applicability of the two notions of subgroup fairness and instantaneous fairness. We use the well-known dataset \cite{angwin2016machine} of estimates of the likelihood of recidivism made by the Correctional Offender Management Profiling for Alternative Sanctions (COMPAS), as used by courts in the United States, cf. Appendix~\ref{sec:COMPAS dataset}. 
The COMPAS dataset, analysed by ProPublica, comprises of defendants' gender, race, age, charge degree, COMPAS recidivism scores, two-year recidivism label, as well as information on prior incidents. 
The COMPAS recidivism scores, ranging from 1 to 10, are positively related to the estimated likelihood of recidivism, given by the COMPAS system.
The two-year recidivism label denotes whether a person actually got rearrested within two years (label 1) or not (label 0). If the two-year recidivism label is $1$, there is also information concerning the recharge degree and the number of days until the person gets rearrested.
The dataset also consists of information on ``Days before Re-offending'', which is the date difference between the defendant's crime offend date and recharge offend date. It could be negatively correlated to the defendant's actual risk level while the COMPAS recidivism scores would be the estimated risk level.
\subsection{An Alternative Approach to COMPAS Dataset}
From the COMPAS dataset cf. Appendix~\ref{sec:COMPAS dataset}, 
we choose $119$ defendants with recidivism label being 1, who are either African-American or Caucasian, male, within the age range of 25-45, and with prior crime counts less than two, with charge degree M and recharge degree M1 or M2.
The defendants are partitioned into two subgroups by their ethnicity and then partitioned by the type of their recharge degree (M1 or M2). Hence, we obtain the $4$ sub-samples.

In the days-to-reoffend-vs-score plot, such as Figure~\ref{fig:COMPASPlot}, dots suggest COMPAS recidivism scores of the four sub-samples against the days before rearrest.
Each curve represents one model, either subgroup-dependent (plotted thin) or Subgroup-Fair (plotted thick).
The thick cyan curve is the race-blind prediction from our Subgroup-Fair method, which equalises scores across the two subgroups.
Ideally, one should like to see smooth, monotonically decreasing curves, overlapping across all subgroup-dependent models. For each sub-sample, the aggregate deviation from the Subgroup-Fair curve would be similar to the aggregate deviations of other sub-samples. 

In Figure~\ref{fig:COMPASPlot}, the dots are far from the ideal monotonically decreasing curve. Furthermore, the subgroup-specific curves (plotted thin) are very different from each other (``subgroup-specific models are unfair''). Specifically, the red and yellow curves are above the sky blue and cornflower blue curves (``at the same risk level, Caucasian defendants get lower COMPAS scores''). 
Notice that the subgroup-dependent models are obtained as follows: we discretise time to $20$-day periods. 
For each subgroup, we check if anyone re-offends within $20$ days (the first period). If so, the (average) COMPAS score (for all cases within the 20 days) is recorded as the observation of the first period of the trajectory of the sub-sample. If not, there is no observation of this period. We repeat this for the subsequent periods and for the three other sub-samples. 



\begin{figure*}[htp]
\centering
\centering{
\includegraphics[width=0.65\textwidth]{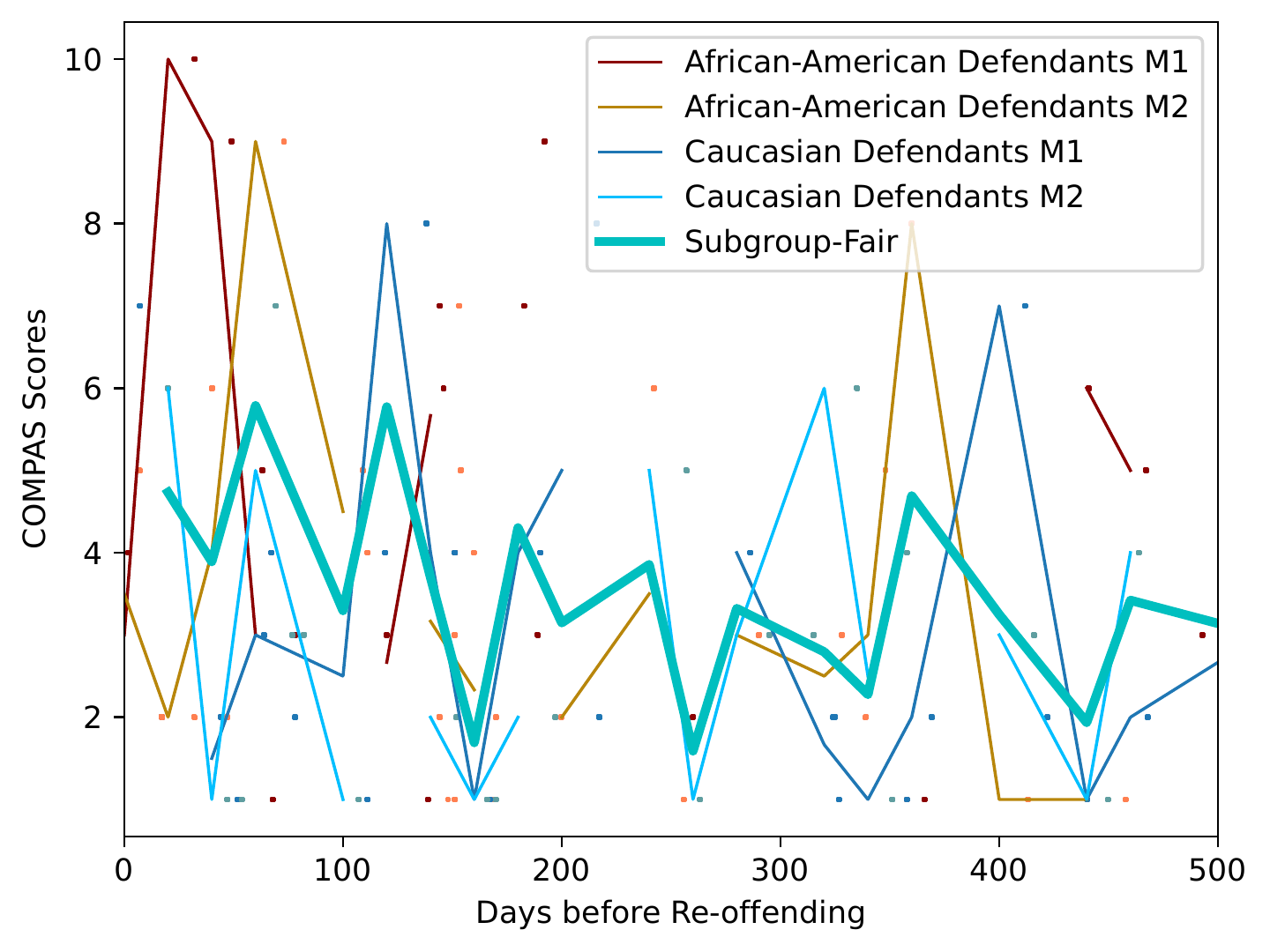}}
\caption{COMPAS recidivism scores of African-American and Caucasian defendants against the actual days before their re-offending. The sample of defendants' scores is divided into four sub-samples based on race and the type of re-offending, distinguished by colours. Dots and curves with the same colour denote the scores of one sub-sample and the trajectory extracted from the scores, respectively. The cyan curve displays the result of ``Subgroup-Fair'' model with $4$ trajectories.}
\label{fig:COMPASPlot}
\end{figure*}

\subsection{A Comparison Against the State of the Art on COMPAS Dataset}
\label{sec:the-State-of-the-Art}

Generally speaking, fairness objectives or constraints might not be easily applied to models that are already in use in applications. For such systems, revision of the model output with some post-processing tools would be a widely applicable and practical solution. 
Since we have shown the existence of unfairness in COMPAS recidivism scores, we now illustrate this approach to improve upon the COMPAS scores by using post-processing methods that embed our fairness notions. We then compare our methods using the AI Fairness 360 toolkit \texttt{AIF360}\footnote{\url{https://github.com/Trusted-AI/AIF360}} of \cite{aif360-oct-2018}.

\paragraph{The training set and test sets: }
The sample set contains 1005 defendants, whose race is either African-American or Caucasian, selected from the first 1200 rows of the COMPAS dataset cf. Appendix~\ref{sec:COMPAS dataset}.
For a single trial, we randomly pick $80\%$ of samples as the training set then test the output on the rest $20\%$ of the samples. Each trial uses a new batch of the training set and the test set generated from the same sample set of 1005 defendants.


Notice that the sample set is biased as there are only 403 Caucasian defendants.
Since existing data may generally contain biases, stemming for example from poor information acquisition process \cite{bertail2021learning}, due to historical and social injustices \cite{ferrer2021bias}, we seek other methods to validate our approaches.
To this end we randomly remove some observations of African-American defendants from the original test set, such that the number of defendants in both subgroups are the same.
The resulted subset is called the re-weighted test set.

\paragraph{Performance indices: }
We use three baseline fairness metrics (i.e., independence, separation, and sufficiency), as well as prediction inaccuracy, to measure the performance of post-processing models.
Essentially, the sample set includes two race subgroups $\mathcal{S}=\{$African-American defendants (AA), Caucasian defendants (C)$\}$. The recidivism label and the prediction of the recidivism label outputted from a model, are denoted by binary variables $Y$ and $f$ respectively. 
Let $P(f=f|Y=Y,s=s)$ be the probability of a defendant from subgroup $s$ with recidivism label $Y$
being predicted to recidivism label $f$.
We set $Y=1$ and $f=1$ to be a defendant re-offending and being predicted to re-offend, thus they are negative events.
Further, we define the indices of three baseline fairness metrics (IND, SP, SF), inaccuracy (INA) and their re-weighted versions (i.e., INDrw, SPrw, SFrw, INArw) in Equation~\eqref{equ:indices}:
\begin{equation}
\begin{split}
    \textrm{IND(rw)}&:=\lvert P(f=1\mid s=\textrm{AA})-P(f=1\mid s=\textrm{C}) \rvert,\\
    \textrm{SP(rw)}&:=\lvert P(f=0\mid Y=1,s=\textrm{AA})-P(f=0\mid Y=1,s=\textrm{C}) \rvert \\
    &\; +\lvert P(f=1\mid Y=0,s=\textrm{AA})-P(f=1\mid Y=0,s=\textrm{C}) \rvert, \\
    \textrm{SF(rw)} &:=\lvert P(Y=1\mid f=1,s=\textrm{AA})-P(Y=1\mid f=1,s=\textrm{C}) \rvert\\
    &\; +\lvert P(f=0\mid Y=0,s=\textrm{AA})-P(f=0\mid Y=0,s=\textrm{C}) \rvert, \\
    \textrm{INA(rw)} &:=P(Y\neq f),
\end{split}
\label{equ:indices}
\end{equation}
where, for example, IND(rw) implies both indices IND and INDrw. The difference between IND and INDrw is that IND measures the performance of a model on the original test set while INDrw on the re-weighted test set. The same applies for SP(rw), SF(rw), INA(rw).
To interpret the definitions in Equation~\eqref{equ:indices}: IND(rw) are race-wise absolute difference of negative rate; SP(rw) combine race-wise absolute difference of false false positive and false negative rates; SF(rw) captures the race-wise absolute difference of positive predictive value and negative predictive value; INA(rw) measure inaccuracy of test set.
In our setting, smaller values of IND(rw), SP(rw), SF(rw), INA(rw) indicate better performance in terms of independence, separation, sufficiency and accuracy, respectively. 

\paragraph{Classification thresholds: }
Since the outputs of all post-processing tools implemented in this paper and COMPAS system are scores from varying intervals and, to transfer the scores to binary labels, we would need a threshold such that $f=1$ when the score, denoted by $g$ is higher than this threshold, and $f=0$ otherwise. 
For ease of comparison, we define uni-race thresholds which differ across different models but all of them are defined as the $x^{\textrm{th}}$ percentile of all scores outputted by the corresponding model, where $x\in[0,100]$ is fixed. 
Notice that there is a gap between the percentage of recidivism in African-American defendants ($46\%$) and Caucasian defendants ($59\%$) in terms of the sample set, and we call those percentages base rates, as in \cite{pleiss2017fairness}.
In fairness to African-American defendants, we introduce race-wise thresholds using base rates: for each model, the percentage of defendants in a subgroup whose scores are higher than the subgroup's threshold needs to be the same as the subgroup's base rate.

\paragraph{Post-processing methods: }

Associated with our fairness notions, we propose two post-processing methods. 
Both methods use simple race-wise linear regression models 
\begin{equation}
g^{(i,s)}=A^{(s)} X^{(i,s)} + e^{(s)}, i\in\mathcal{I}^{(s)}, s\in\mathcal{S},
\label{equ:post-process-constraints}
\end{equation}
where the subscript $t$ is removed such that we only consider prediction in one period. In other words, we cast the problem of prediction into classification.
$A^{(s)}$ concatenates the regression coefficients. $X^{(i,s)}$ concatenates explanatory variables, including COMPAS recidivism score, prior incidents (i.e., the sum of ``prior counts'', ``juv\_ fel\_ count'' and ``juv\_ misd\_ count''), age category (i.e., 1 if age is less than 25 and 0 otherwise), and recidivism label.
$e^{(s)}$ corresponds to a noise to the linear relationship.
$g^{(i,s)}$ is the post-processed recidivism score of the defendant $i$ in subgroup $s$, and $Y^{(i,s)}$ is the actual recidivism label (i.e., the ground truth).
Let $\loss^{(i,s)}(g):=\|Y^{(i,s)}-g^{(i,s)}\|$, our post-processing methods are Equation~\eqref{equ:post-process-formulations} subject to Equation~\eqref{equ:post-process-constraints}, with $\lambda_3=0.05$. Further, the score $g^{(i,s)}$ would be mapped to the binary prediction of recidivism label $f^{(i,s)}$ using a threshold. 
\begin{equation}
\begin{array}{rl}
\textrm{\textbf{Subgroup-Fair}}
&\min_{g,A,e} \left \{ \max_{s\in\mathcal{S}} \left \{
\frac{1}{\lvert\mathcal{I}^{(s)}\rvert}
\sum_{i \in \mathcal{I}^{(s)}} \loss^{(i,s)}(g) \right \} + \lambda_3 \sum_{s\in\mathcal{S}}  \left(e^{(s)}\right)^2 \right\}\\
\textrm{\textbf{Instant-Fair}}
&\min_{g,A,e}\left \{ \max_{i\in\mathcal{I}^{(s)},s\in\mathcal{S}} \left \{ \loss^{(i,s)}(g) \right \}  + \lambda_3 \sum_{s\in\mathcal{S}}  \left(e^{(s)}\right)^2\right \}
\end{array}
\label{equ:post-process-formulations}
\end{equation}

\paragraph{}
In Figure~\ref{fig:AIF360}, we test the performance of all post-processing methods implemented in AI Fairness 360 toolkit: 
\begin{itemize}
    \item ``AIF360'': calibrated equalised odds post-processing with cost constraint being a combination of both false negative rate and false positive rate, as suggested by the authors of the AI Fairness 360 toolkit \cite{aif360-oct-2018},
    \item ``CaliEqOdds(fnr)'': calibrated equalised odds post-processing with cost constraint being the false negative rate,
    \item ``CaliEqOdds(fpr)'': calibrated equalised odds post-processing with cost constraint being the false positive rate,
    \item ``EqOdds'': equalised odds post-processing,
    \item ``RejectOption'': reject option classification.
\end{itemize}
Note that ``AIF360'', ``CaliEqOdds(fnr)'', ``CaliEqOdds(fpr)'' are based on the fairness notion of ``calibrated equalised odds'' in \cite{pleiss2017fairness}. 
``EqOdds'' is derived from the fairness notion of ``equalised odds'' in \cite{hardt2016equality}.
``RejectOption'' comes from \cite{kamiran2012decision}, which is rooted in the fairness notion of ``demographic parity'' \cite{calder2009experimental}.
Those five methods are implemented in five trials for each of three uni-race thresholds $x=[47, 53, 60]$ ($5\times 5\times 3$ runs). 
The left subplot displays mean values of eight indices across all five trials and three thresholds, with each angular axis representing one index, and each colour denoting one post-processing method.
The right subplot represents the values of all experimental runs as dots in a circular sector, with each sector representing one index.
Each sector is labelled with the index immediately counter-clockwise to it. For instance, the sector between the labels ``IND'' and ``INDrw'' displays the ``IND'' values of all experimental runs.
In both subplots, the value represented by a dot is displayed by its distance from the original point, with shorter distances indicating better performance.

Figure~\ref{fig:AIF360} depicts a summary of the state of the art on COMPAS dataset, and how we select the appropriate method to benchmark our own algorithms.
Referring to this figure, since in both subplots, most of yellow (``AIF360'') dots are relatively closer to the origin than other dots, it seems fair to consider ``AIF360'' as the state of the art post-processing method, at least within those implemented in AI Fairness 360 toolkit, and to compare our methods against it in the following.

In Figure~\ref{fig:postprocess}, COMPAS scores (``COMPAS'' red), the state of the art (``AIF360'', yellow), and the outputs of our methods ``Subgroup-Fair'' (blue) and ``Instant-Fair'' (green) are evaluated across 50 trials using base rates as race-wise thresholds ($50\times 4$ runs).
The left subplot illustrates the average performance of four methods, where dots represent the mean values of original indices, and bars are those of re-weighted indices.
The right subplot displays fairness performance of all experimental runs in a triangular area.
For a single run, the original fairness metrics (i.e., IND, SP, SF) are denoted by one dot and re-weighted ones (i.e., INDrw, SPrw, SFrw) are shown as one square. 
A marker, that is, a dot or a square, represents the value of IND(rw), SP(rw), SF(rw), by its positions along the left, right, and bottom axes in ternary coordinates.
The colour of this marker denotes the method used in this run.

Figure~\ref{fig:postprocess} illustrates the performance of our methods compared with the state of the art and COMPAS scores when using base rates as race-wise thresholds. 
As we can see on the left, there is not much difference between the values of original indices and their re-weighted versions, except that ``Instant-Fair'' generally performs worse in re-weighted version than in original version. 
It implies that ``Instant-Fair'' might not be appropriate to use in this case because its performance varies with the test set being re-weighted or not.
The performance of ``Subgroup Fair'' is similar to that of ``COMPAS'', but with a slight improvement in IND(rw) and SP(rw).
``AIF360'' seems to sacrifice a lot of accuracy for lower average fairness indices, while its fairness performance shows a lot of variability, as shown in  the right subplot.
On the contrary, the concentration of blue and green markers (i.e., dots and squares) indicates less variability of our methods.

In Figure~\ref{fig:thresholds}, 
COMPAS scores (``COMPAS'' red), the state of the art (``AIF360'', yellow), and the outputs of our methods ``Subgroup-Fair'' (blue) and  ``Instant-Fair'' (green) are evaluated across 50 trials, with 10 different uni-race thresholds $x=[20,27,\dots,80]$ ($50\times 10\times 4$ runs).
Each subplot represents the mean (curves) and mean $\pm$ one standard deviations (shaded error bands) of the corresponding index across 50 trials, against 10 different uni-race thresholds $x=[20,27,\dots,80]$
with four methods distinguished by the same palette as in Figure~\ref{fig:postprocess}.

Figure~\ref{fig:thresholds} depicts an investigation of the performance of ``Subgroup-Fair'', ``Instant-Fair'', ``AIF360'' and ``COMPAS'' for different uni-race thresholds. 
Notice that the error bands generally overlap each other in subplots of the first and third rows (which represent the indices IND(rw) and SF(rw)). 
One potential implication of this work is that there might not be significant differences amongst the four methods, in terms of IND(rw) and SF(rw), when using uni-race thresholds. 
If we look at the remaining indices, ``Subgroup-Fair'' (blue) surpasses ``COMPAS'' (red) in terms of the performance indexes SP(rw), and both achieve the best in INA(rw).
Furthermore, ``AIF360'' has relatively low values of SP(rw), but at a large expense of INA(rw). 
\begin{figure*}[tbp]
\includegraphics[width=0.49\textwidth]{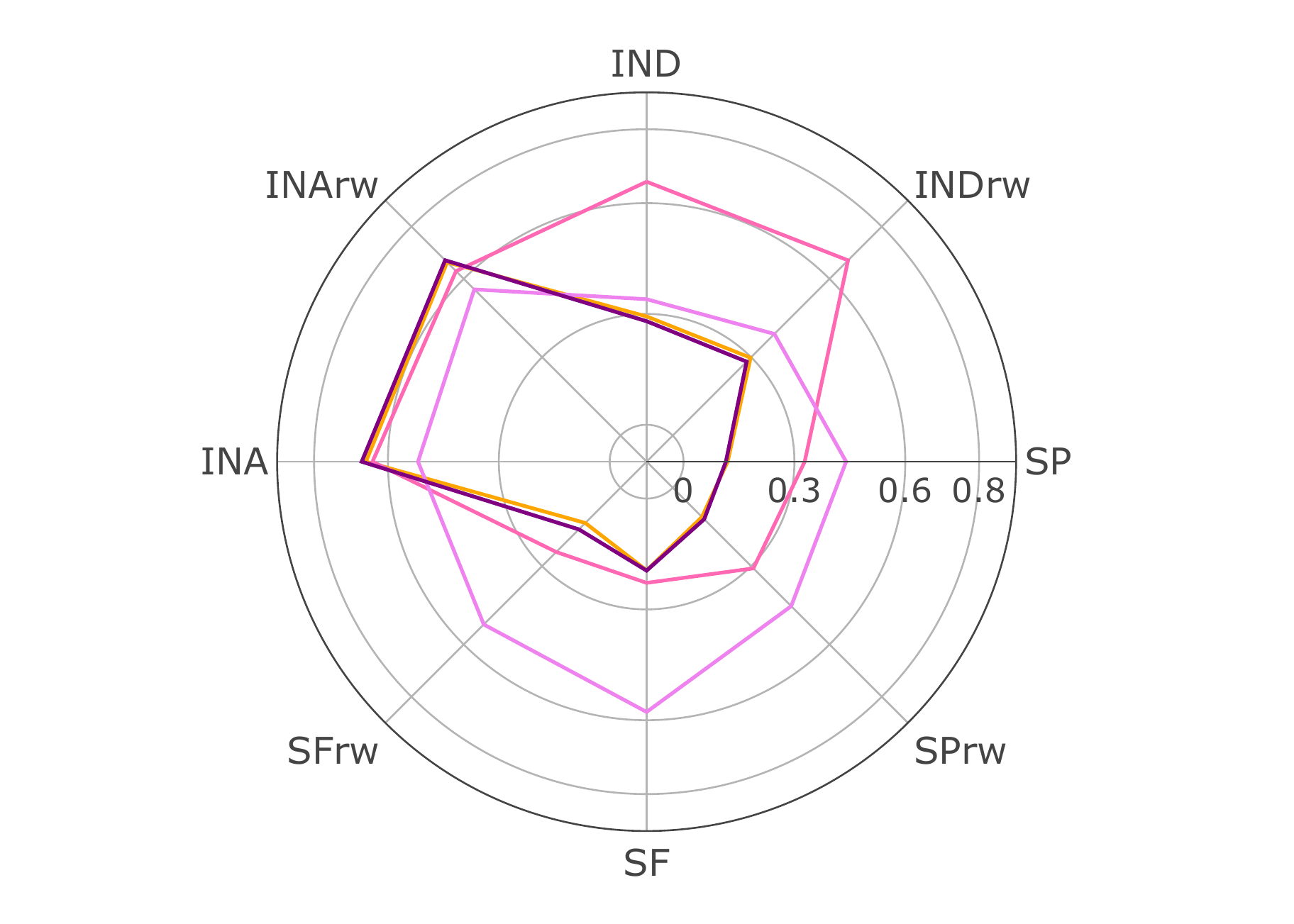}
\includegraphics[width=0.49\textwidth]{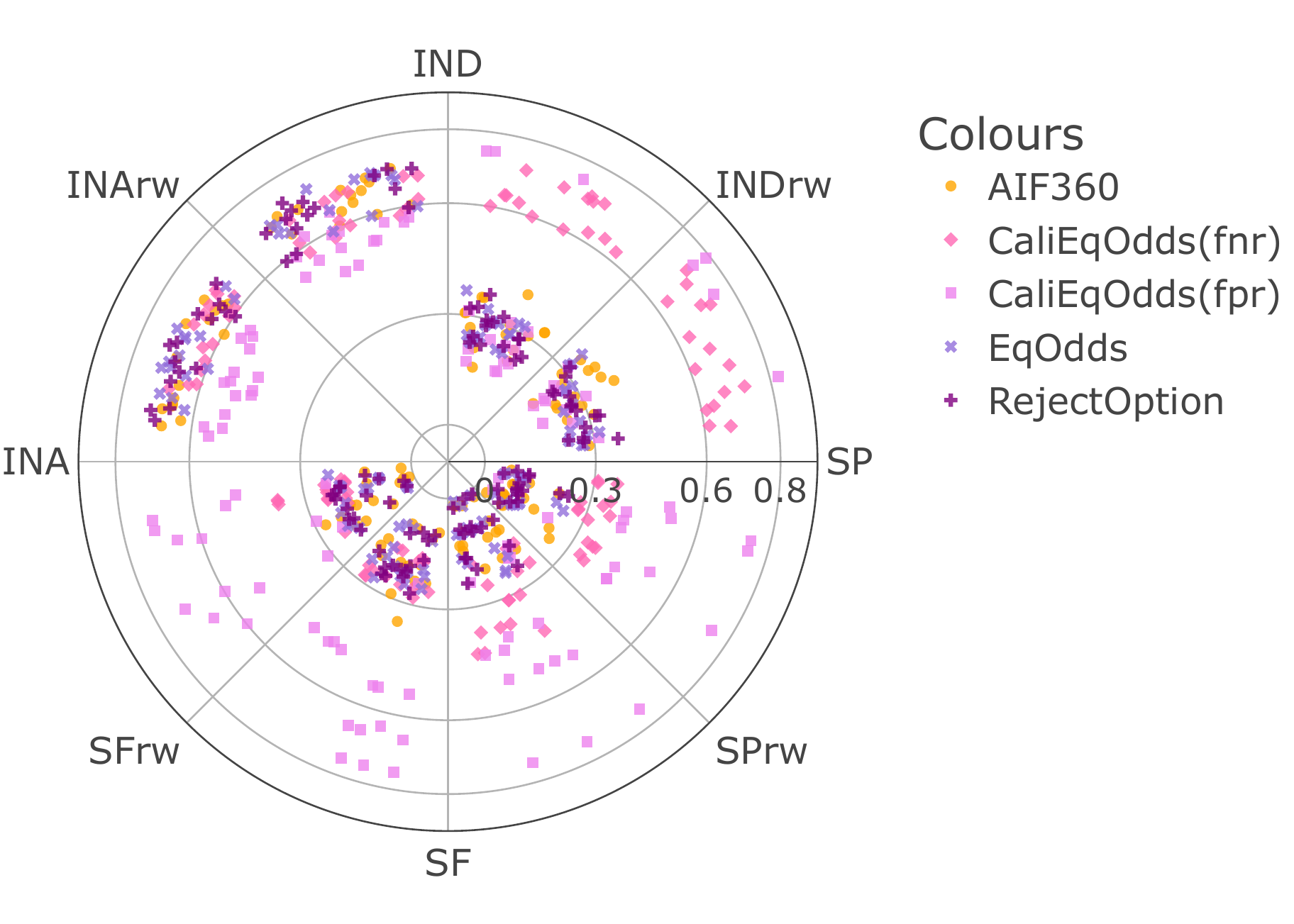}
\caption{
The state of the art in post-processing for improving fairness. 
Five post-processing methods implemented in AI Fairness 360 toolkit, i.e., ``AIF360'' (yellow), ``CaliEqOdds(fnr)'' (pink), ``CaliEqOdds(fpr)'' (violet), ``EqOdds'' (light purple), and ``RejectOption'' (purple), are run in five trials per each of three uni-race thresholds $x=[47, 53, 60]$ ($5\times 5\times 3$ runs).
Left: The mean values of eight indices on eight angular axes. 
Right: 
Each sector is labelled with the index immediately counter-clockwise to it.
Dots in one sector denote the values of the corresponding index of all experimental runs.}
\label{fig:AIF360}
\end{figure*}

\begin{figure*}[tbp]
\centering
\begin{tabular}{p{0.4\textwidth} p{0.6\textwidth}} \vspace{0pt}\includegraphics[width=0.39\textwidth]{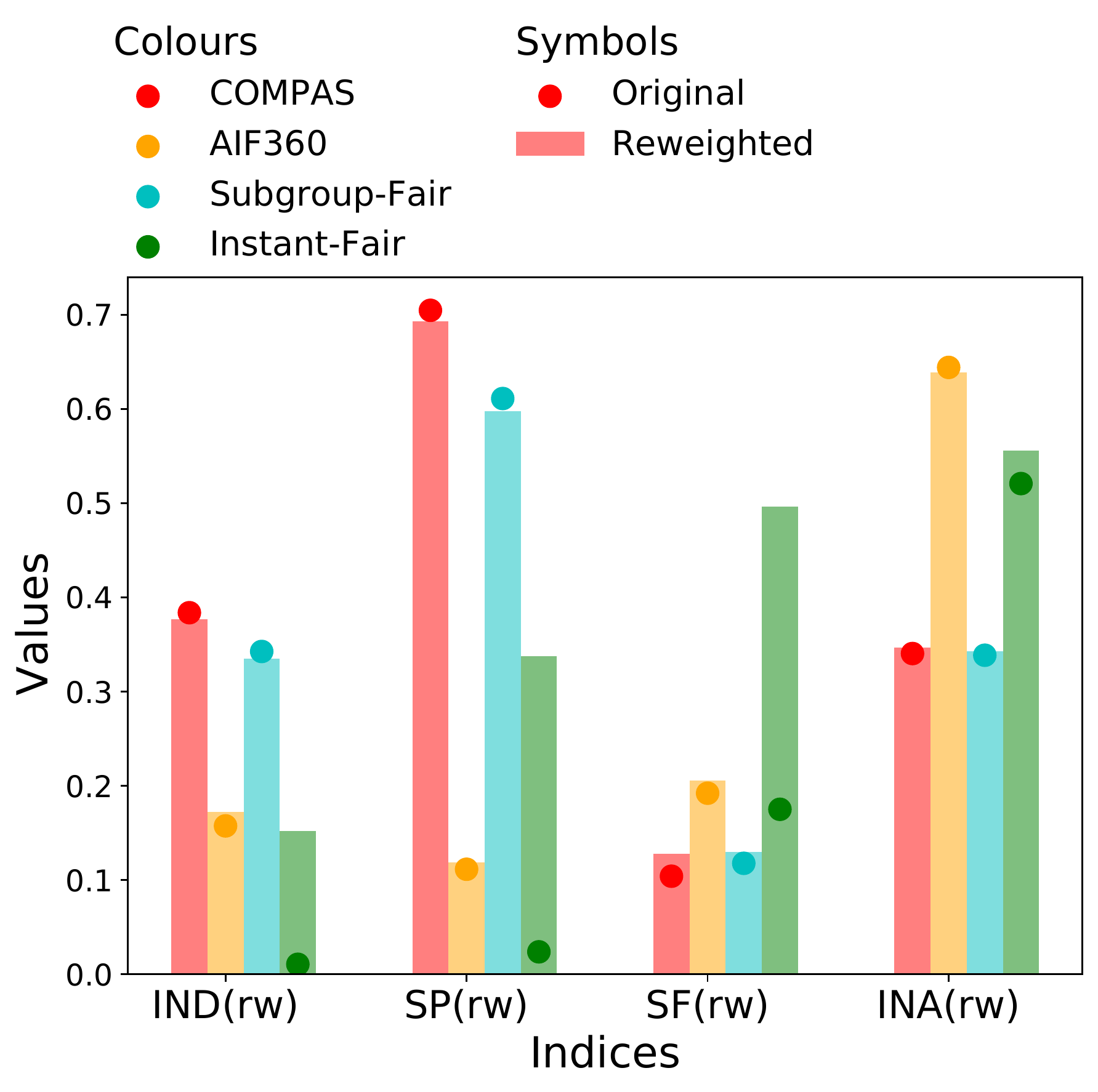} & \vspace{0pt}\includegraphics[width=0.6\textwidth]{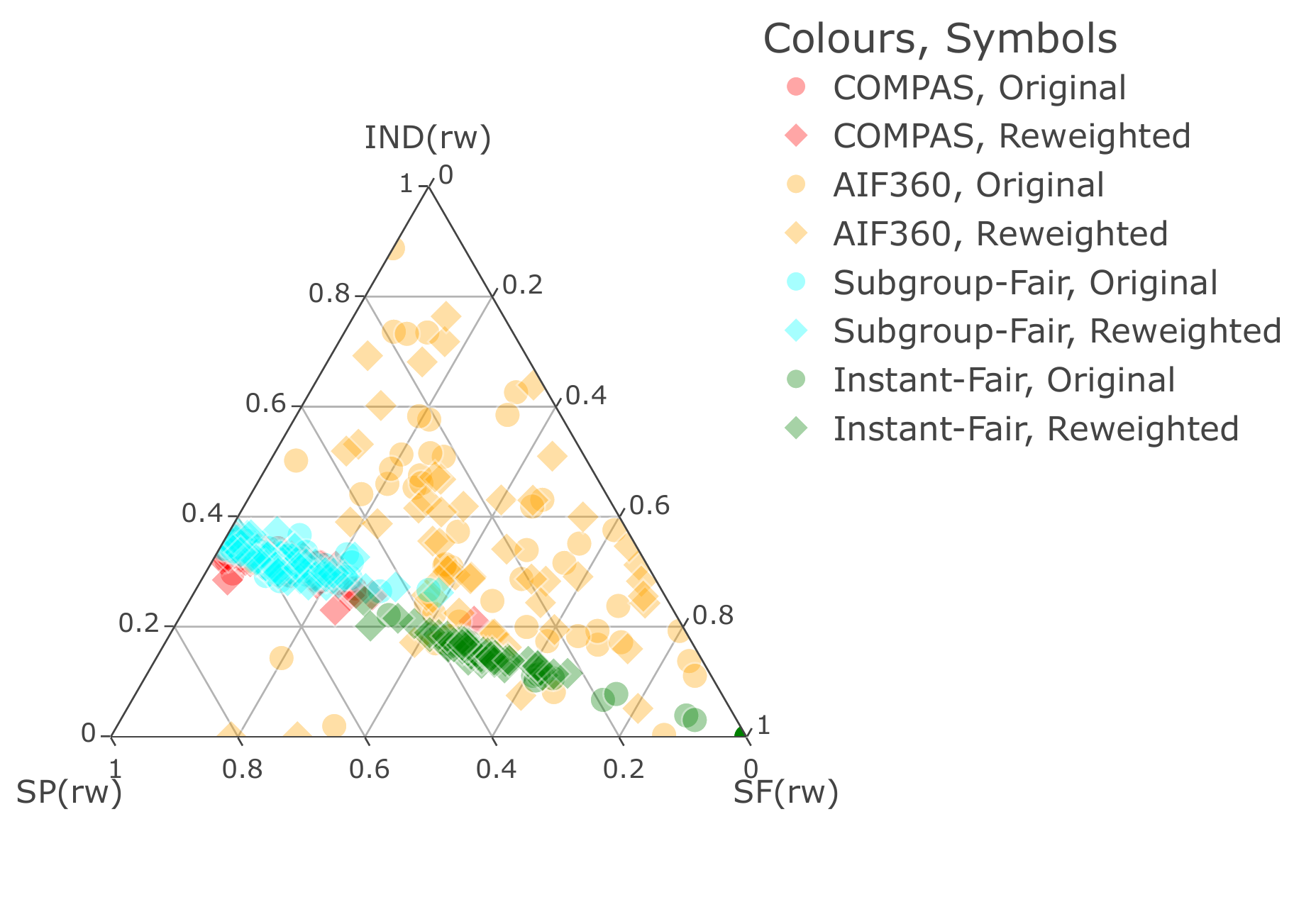} 
\end{tabular}
\caption{The effects of our methods, when applied as post-processing. Three post-processing methods ``subgroup-fair'' (blue), ``instant-fair'' (green), ``AIF360'' (yellow) are run in 50 trials, and their outputs are compared with the COMPAS scores (red) with race-wise thresholds being base rates.
Left: The mean values of original (dots) and re-weighted (bars) indices associated with outputs from four models.
Right: The original (dots) and re-weighted (squares) fairness metrics (i.e., INDrw, SPrw, SFrw) of all experimental runs.}
\label{fig:postprocess}
\end{figure*}

\begin{figure*}[tbp]
\centering
\includegraphics[width=0.8\textwidth]{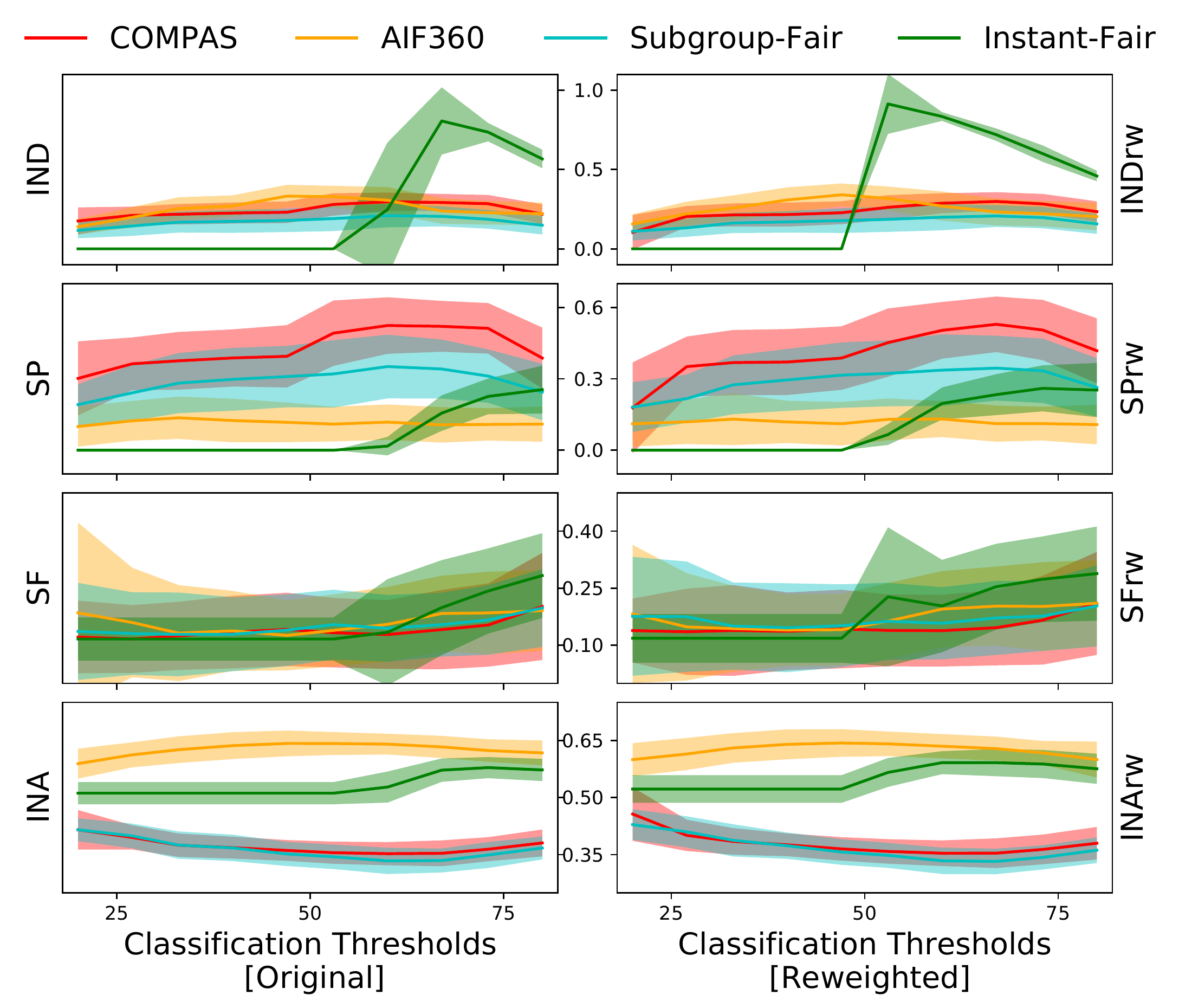}
\caption{Three post-processing methods ``subgroup-fair'' (blue), ``instant-fair'' (green), ``AIF360'' (yellow) are run in 50 trials, and their outputs are compared with the COMPAS scores (red), with 10 uni-race thresholds $x=[20,26.7,\dots,80]$.
Each subplot displays the mean (curves) and mean $\pm$ one standard deviations (shaded error bands) of the corresponding index across the 50 trials, against 10 uni-race thresholds.
}
\label{fig:thresholds}
\end{figure*}


\clearpage
\section{Conclusions}
We have introduced the two natural notions of fairness in forecasting. 
When the corresponding optimisation problems are solved to global optimality, the solutions outperform the COMPAS system in terms of independence and separation indices IND(rw) and SP(rw), cf. Equation~\eqref{equ:indices}.

As a further technical contribution, we have presented globally convergent methods for solving the optimisation problems arising from the two notions of fairness using hierarchies of convexifications of non-commutative polynomial optimisation problems. Also, we have shown that the run-time of standard solvers for the convexifications is independent of the dimension of the hidden state. This provides a technical tool in machine learning and statistics that is of independent interest, and that can also be applied in other settings.


\acks{Quan Zhou's and Robert Shorten's work has been supported by the Science Foundation Ireland under Grant 16/IA/4610. Jakub Mare\v{c}ek acknowledges support of the OP RDE funded project CZ.02.1.01/0.0/0.0/16\_019/0000765 ``Research Center for Informatics''.
This work has received funding from the European Union’s Horizon Europe research and innovation programme under grant agreement No. 101070568. This work was also supported by Innovate UK under the Horizon Europe Guarantee; UKRI Reference Number: 10040569 (Human-Compatible Artificial Intelligence with Guarantees (AutoFair)).
}

\appendix
\section{Motivation}
\subsection{Insurance Pricing}
\label{sec:insurancepricing}

Let us consider two motivating examples.
One important application arises in Actuarial Science. In the European Union, a directive (implementing the principle of equal treatment between men and women in the access to and supply of goods and services), bars insurers from using gender as a factor in justifying differences in individuals' premiums. 
In contrast, insurers in many other territories classify insureds by gender, because females and males have different behaviour patterns, which affects insurance payments. Take the annuity-benefit scheme for example. It is a well-known fact that females have a longer life expectancy than males \cite{huang2020effect}. The insurer will hence pay more to a female insured during her lifetime, compared to a male insured, on average \cite{thiery2006fairness}. Because of the directive, a unisex mortality table needs to be used. As a result, male insureds receive less benefits, while paying the same premium in total as the female subgroup \cite{thiery2006fairness}. Consequently, male insureds might leave the annuity-benefit scheme (known as adverse selection), which makes the unisex mortality table more challenging to use in the estimation of the life expectancy of the ``unisex'' population, where female insureds become the advantaged subgroup.

Consider a simple actuarial pricing model of annuity insurance. Insureds enter an annuity-benefit scheme at time $0$ and each insured can receive 1 euro at the end of each year for at most 10 years on the condition that it is still alive. Let $p_t$ denotes how many insureds left in the scheme in the end of the $t^{th}$ year. Suppose there are $p_0$ insureds in the beginning and the pricing interest rate is $i$ $(i\leq 1)$. The formula of calculating the pure premium is in Equation~\eqref{equ:premium}, thus summing up the present values of payment in each year and then divided by the number of insureds in the beginning.

\begin{equation*}
    \premium := \frac{\sum_{t=1}^{10} p_t \times (1+i)^{-t} }{p_0} \label{equ:premium}
\end{equation*}

The most important quality $p_t$ is derived from estimating insureds' life expectancy. 
Suppose the insureds can be divided into female and male subgroups. Each subgroup has one trajectory: $\{Y_t\}^{(\ \cdot \ ,f)}$ for female subgroup, $\{Y_t\}^{(\ \cdot \ ,m)}$ for male subgroup for $1\leq t\leq 10$, where the superscript $i$ is dropped. The two trajectories indicate how many female and male insureds are alive at the end of the $t^{th}$ year, respectively. Both trajectories can be regarded as linear dynamic systems. We have

\begin{eqnarray*}
    Y_t^{(\ \cdot \ ,f)} =& G^{(f)} Y_{t-1}^{(\ \cdot \ ,f)} + \omega_t^{(f)}, & 2\leq t\leq 10, 
    \\
    Y_t^{(\ \cdot \ ,m)} =& G^{(m)} Y_{t-1}^{(\ \cdot \ ,m)} + \omega_t^{(m)}, & 2\leq t\leq 10, %
    \\
    p_t =& Y_t^{(\ \cdot \ ,f)}+Y_t^{(\ \cdot \ ,m)}, & 1\leq t\leq 10, 
\end{eqnarray*}

where $\omega_t^{(f)}$ and $\omega_t^{(m)}$ are measurement noises while $G^{(f)}$ and $G^{(m)}$ are system matrices for female LDS $\mathcal{L}^{(f)}$ and male LDS $\mathcal{L}^{(m)}$ respectively. Note that these are state processes, without any observation process: the number of survivals can be precisely observed. 
To satisfy the directive, one needs to consider a unisex model:

\begin{eqnarray*}
    f_t &=& G f_{t-1} + \omega_t, \, 2\leq t\leq 10,\label{equ:annuity-female+male-LDS}
\end{eqnarray*}

where $2\leq t\leq 10$ and $\omega_t$ and $G$ pertain to the unisex insureds LDS $\mathcal{L}$. Subsequently, the loss functions for female (f) and male (m) subgroups are:

\begin{eqnarray*}
\loss^{(\ \cdot \ ,f)}(f_t) :=& ||Y_t^{(\ \cdot \ ,f)}-f_t|| &,1\leq t\leq 10, 
\\
\loss^{(\ \cdot \ ,m)}(f_t) :=& ||Y_t^{(\ \cdot \ ,m)}-f_t|| &,1\leq t\leq 10, 
\end{eqnarray*}

Since the trajectories $\{Y_t\}^{(\ \cdot \ ,f)}$ and $\{Y_t\}^{(\ \cdot \ ,m)}$ have the same length and there is only one trajectory in each subgroup, the two objective Eq.(2)-Eq.(3)has the form:

\begin{equation*}
        \min_{f_t,1\leq t\leq 10} \max \left \{ 
        \sum_{t=1}^{10} \loss^{(\ \cdot \ ,f)}(f_t),
        \sum_{t=1}^{10} \loss^{(\ \cdot \ ,m)}(f_t) \right \}
\end{equation*}
\begin{equation*}
        \min_{f_t,1\leq t\leq 10} \left \{
        \max_{1\leq t\leq 10, s\in\{f,m\}} \left \{
        \loss^{(\ \cdot \ ,s)}(f_t) \right \} \right \}
\end{equation*}

\subsection{Personalised Pricing}
\label{sec:Personalised Pricing}
Another application arises in personalised pricing (PP). 
For example, Amazon has been found \cite{OECDbackground} to sell certain products to regular consumers at higher prices. This is legal, albeit questionable.
In contrast, gender-based price discrimination 
\cite{abdou2019gender}
 violates \cite{OECDbackground}  anti-discrimination laws in many jurisdictions. 

Let us consider an idealised example of PP: Consider a soap retailer, whose customers contain female and male subgroups. Each gender has a specific dynamic system modelling its willing to pay (``demand price'' of each subgroup), while the retailer should set a ``unisex'' price. As in the discussion of insurance pricing, we consider subgroups $S=\{$female, male$\}$ and use superscripts $(f),(m)$ to distinguish the related quantities. Unlike in insurance pricing, the demand price of each customer is regarded as a single trajectory. More importantly, since customers might start buying soap, quit buying the soap, or move to other substitutes at different time points, those trajectories of demand prices are assumed to be of varying lengths. For example, a customer starts to buy the soap at time $3$ but decides to buy hand wash instead from time $7$. 

Let us assume there are $\lvert\mathcal{I}^{(f)}\rvert$ female customers and $\lvert\mathcal{I}^{(m)}\rvert$ customers in the overall time window $\mathcal{T}^+$.
Let $Y_t^{(i,s)}$ denote the estimated demand price at time $t$ of the $i^{th}$ customer in subgroup $s$. These evolve as:

\begin{eqnarray*}
    \phi_{t}^{f} =& G^{(f)} \phi_{t-1}^{(f)} + \omega_t^{(f)} &, t\in\mathcal{T}^+, 
    \\
    Y_{t}^{(i,f)} =& F^{(f)'} \phi_{t}^{(f)} + \nu_t^{(i,f)} &, t\in\mathcal{T}^{(i,f)}, i\in\mathcal{I}^{(f)}, 
    \\
    \phi_{t}^{m} =& G^{(m)} \phi_{t-1}^{(m)} + \omega_t^{(m)} &, t\in\mathcal{T}^+, 
    \\
    Y_{t}^{(i,m)} =& F^{(m)'} \phi_{t}^{(m)}+ \nu_t^{(i,m)} &,t\in\mathcal{T}^{(i,m)}, i\in\mathcal{I}^{(m)}.
\end{eqnarray*}

The unisex model for demand price considers the unisex state $m_t$, the unisex system matrices $G,F$, and unisex noises $\omega_t,\nu_t$:
\begin{eqnarray*}
    m_t =& G m_{t-1} + \omega_t &, \, t\in\mathcal{T}^+, 
    \\
    f_t =& F' m_{t} + \nu_t &, \, t\in\mathcal{T}^+. 
\end{eqnarray*}
For $\loss^{(i,f)}(f_t) := ||Y_t^{(i,f)}-f_t|| ,t\in\mathcal{T}^{(i,f)}, i\in\mathcal{I}^{(f)}$ and $\loss^{(i,m)}(f_t) := ||Y_t^{(i,m)}-f_t||,t\in\mathcal{T}^{(i,m)}, i\in\mathcal{I}^{(m)}$, the two objectives
Eq.(2)-Eq.(3) have the form:

\begin{equation*}
        \min_{f_t,t\in\mathcal{T}^+} 
        \max_{s\in\mathcal{S}} \left \{
        \frac{1}{\lvert\mathcal{I}^{(s)}\rvert}\sum_{i=1}^{\mathcal{I}^{(s)}} \frac{1}{\lvert\mathcal{T}^{(i,s)}\rvert}\sum_{t\in\mathcal{T}^{(i,s)}} \loss^{(i,s)}(f_t) \right \}
\end{equation*}

\begin{equation*}
        \min_{f_t,t\in\mathcal{T}^+} \left \{ \max_{t\in\mathcal{T}^{(i,s)},i\in\mathcal{I}_s,s\in\mathcal{S}} \left \{ \loss^{(i,s)}(f_t) \right \} \right \}
\end{equation*}

We also refer to \cite{dong2020protecting} for further work on protecting customers' interests in personalised pricing via fairness considerations.

\section{Background}
\label{sec:background}
In this paper, we consider the case of multiple variants of the LDS and conduct proper learning of the LDS in a way of fairness using the technologies of non-commutative polynomial optimisation.
In Section~\ref{sec:background}, we firstly set our work in the context of system identification and control theory.
Secondly, we introduce the concept of fairness, which can be used to deal with multiple variants of the LDS.
In the end of this section, we provide a brief overview of 
non-commutative polynomial optimisation, pioneered by 
\cite{Pironio2010} and nicely surveyed by 
\cite{burgdorf2016optimization}, which is our key technical tool.
\subsection{Related Work in System Identification and Control}
\label{sec:relwork}
Research within System Identification variously appears in venues associated with Control Theory, Statistics, and Machine learning. We refer to \cite{Ljung1999} and \cite{tangirala2014principles} for excellent overviews of the long history of research in the field, going back at least to \cite{ASTROM1965}. In this section, we focus on pointers to key more recent publications. 
In improper learning of LDS, a considerable progress has been made in the analysis of predictions for the expectation of the next measurement using auto-regressive (AR) processes. In \cite{anava13}, first guarantees were presented for auto-regressive moving-average (ARMA) processes. In \cite{liu2016online}, these results were extended to a subset of autoregressive integrated moving average (liu2016online) processes.  \cite{Jakub} have shown that up to an arbitrarily small error given in advance, AR($s$) will perform as well as \emph{any} Kalman filter on any bounded sequence. 
This has been extended by \cite{tsiamis2020online} to Kalman filtering with logarithmic regret.
Another stream of work within improper learning focuses on sub-space methods \cite{katayama2006subspace,van2012subspace} and spectral methods. 
\cite{tsiamis2019sample,tsiamis2019finite} presented the present-best guarantees for traditional sub-space methods.
Within spectral methods, \cite{hazan2017learning} and \cite{hazan2018spectral} have considered learning LDS with input, employing certain eigenvalue-decay estimates of Hankel matrices in the analyses of an auto-regressive process in a 
dimension increasing over time.
We stress that none of these approaches to improper learning  are ``prediction-error'': They do \emph{not} estimate the system matrices.

In proper learning of LDS, many state-of-the-art approaches consider the least-squares method, despite complications encountered in unstable systems \cite{faradonbeh2018finite}. \cite{simchowitz2018learning} have  
provided non-trivial guarantees for the ordinary least-squares (OLS) estimator 
in the case of stable $G$ and there being no hidden component, i.e., $F'$ being an identity and $Y_t = \phi_t$. 
Surprisingly, they have also shown that more unstable linear systems are easier to estimate than less unstable ones, in some sense. 
\cite{simchowitz2019learning} extended the results to allow for a certain pre-filtering procedure.
\cite{SarkarRakhlin} extended the results to cover stable,
marginally stable, and explosive regimes.

Our work could be seen as a continuation of the least squares method to processes with hidden components, with guarantees of global convergence.
In Computer Science, our work could be seen as an approximation scheme \cite{vazirani2013approximation},
as it allows for $\epsilon$ error for any $\epsilon > 0$.



\subsection{Learning from Imbalanced Data}

Traditional machines learning algorithms can be biased towards majority class over-prevalence \cite{chawla2003c4}, i.e., the under-representation bias \cite{blum2019recovering}. Also, the cost of mis-classifying an abnormal event (minority class) as a normal event (majority class) is often relatively high \cite{chawla2008automatically,chawla2002smote}. For example, in the case of fraud, diseases, those cases are rare but able to cause serious damages, so it is of great interest to research. The benchmark of learning from imbalanced data was pioneered by \cite{chawla2002smote}. They proposed the Synthetic Minority Over-sampling Technique (SMOTE), such that a combination of over-sampling the minority class and under-sampling the majority class can efficiently improve the classifier performance.

The research in learning from imbalanced data has been extensively studied with a particular focus on classification and other predictive contexts as many real-world applications are already facing this problem \cite{torgo2017learning,fernandez2018smote}. SMOTE has been successfully extended to a variety of applications because of its simplicity and robustness \cite{cieslak2008learning}. Surprisingly, \cite{chawla2008automatically} provides an algorithm that automatically discovers the amount of re-sampling. 
One the other hand, \cite{moniz2017resampling} proposed the concept of temporal and relevance bias in extension of re-sampling strategies. For the clear journey of SMOTE, please refer to \cite{fernandez2018smote}.

Unlike the common solution of re-sampling, we address the under-representation bias from the view of optimisation, such that the ``loss'', or other statistical performance is equalised over majority and minority subgroups.

\subsection{Non-Commutative Polynomial Optimisation}
\label{sec:ncpop}

In learning of the LDS, the key technical tool of this paper is non-commutative polynomial optimisation (NCPOP),
first introduced by \cite{Pironio2010}. 
Here, we provide a brief summary of their results, and refer to \cite{burgdorf2016optimization} for a book-length introduction.
NCPOP is an operator-valued optimisation problem with a standard form in Equation~\eqref{NCPO}:
\begin{mini}
{(\mathcal{H},\mathbf{X},\phi)} {\langle\phi,p(\mathbf{X})\phi\rangle}{P:}{p*=}
\addConstraint{q_i(\mathbf{X})}{\succcurlyeq 0, }{i=1,\ldots,m}
\addConstraint{\langle\phi,\phi\rangle}{= 1,}{}\label{NCPO}
\end{mini}
where $\mathbf{X}=(X_1,\dots,X_n)$ is a $n$-tuple of bounded operators on a Hilbert space $\mathcal{H}$ in this section.
The normalised vector $\phi$, i.e., $\|\phi \|^2=1$ is also defined on $\mathcal{H}$ with inner product $\langle\phi,\phi\rangle$ equals to $1$. 
$p(\mathbf{X})$ and $q_i(\mathbf{X})$ are polynomials 
and $q_i(\mathbf{X})\succcurlyeq 0$ denotes that the operator $q_i(\mathbf{X})$ is positive semi-definite. 

In contrast to traditional scalar-valued, vector-valued, or matrix-valued optimisation techniques, the dimension of operators $\mathbf{X}$ is unknown \textit{a priori}.
Let $[\mathbf{X},\mathbf{X}^{\dag}]$ denotes these $2n$ operators, with the $\dag$-algebra being conjugate transpose.
Monomials $\omega,\mu$ and $\nu$ in following text are products of powers of variables from $[\mathbf{X},\mathbf{X}^{\dag}]$.
The degree of a monomial, denoted by $|\omega|$, refers to the sum of the exponents of all operators in the monomial $\omega$.
Let $\mathcal{W}_k$ denote the collection of all monomials whose degrees $|\omega|\leq k$.
Polynomials $p(\mathbf{X})$ and $q_i(\mathbf{X})$ of degrees $\deg(p)$ and $\deg(q_i)$, respectively,
can be written as:

\begin{equation*}
p(\mathbf{X})=\sum_{\lvert\omega\rvert\leq \deg(p)} p_{\omega} \omega,\quad
q_i(\mathbf{X}) = \sum_{\lvert\mu\rvert\leq \deg(q_i)} q_{i,\mu} \mu, 
\label{LCoP}
\end{equation*}
where $i = 1,\ldots,m$. 
Following \cite{akhiezer1962some}, we can define the moments on field $\mathbb{R}$ or $\mathbb{C}$, with a feasible solution $(\mathcal{H},\mathbf{X},\phi)$ of problem in Equation~\eqref{NCPO}:
\begin{equation*}
y_{\omega} = \langle \phi, \omega \phi \rangle, \label{DEFoMOMENT}
\end{equation*}
for all $\omega \in \mathcal{W}_{\infty}$ and  $y_1=\langle \phi,\phi \rangle=1$.
Given a degree $k$, the moments whose degrees are less or equal to $k$ form a sequence of $y=(y_{\omega})_{\lvert\omega\rvert \leq 2k}$.
With a finite set of moments $y$ of degree $k$, we can define a corresponding $k^{th}$ order moment matrix $M_k(y)$:
\begin{equation*}
M_k(y)(\nu,\omega) = y_{\nu^{\dag}\omega} = \langle \phi, \nu^{\dag}\omega \phi \rangle,
\label{equ:moment-matrix}
\end{equation*}
for any $ \lvert\nu\rvert,\lvert\omega\rvert \leq k$ and a localising matrix $M_{k-d_i}(q_i y)$: 
\begin{align*}
M_{k-d_i}(q_iy)(\nu,\omega) & = \sum_{\lvert\mu\rvert \leq \deg(q_i)} q_{i,\mu} y_{\nu^{\dag}\mu\omega} \\ & = \sum_{\lvert\mu\rvert \leq \deg(q_i)} q_{i,\mu} \langle \phi, \nu^{\dag} \mu\omega \phi \rangle, \notag
\end{align*}
for any $\lvert\nu\rvert,\lvert\omega\rvert \leq k-d_i$, where $d_i=\lceil \deg(q_i)/2\rceil$. The upper bounds of $\lvert\nu\rvert$ and $\lvert\omega\rvert$ are lower than that of moment matrix because $y_{\nu^{\dag}\mu \omega}$ is only defined on $\nu^{\dag}\mu\omega \in \mathcal{W}_{2k}$ while $\mu\in \mathcal{W}_{\deg(q_i)}$.

If $(\mathcal{H},\mathbf{X},\phi)$ is feasible, 
one can utilise the Sums-of-Squares theorem of \cite{helton2002positive} and \cite{mccullough2001factorization} to derive semidefinite programming (SDP) relaxations.
In particular, we can obtain a $k^{th}$ order SDP relaxation of the non-commutative polynomial optimisation problem in Equation~\eqref{NCPO} by choosing a degree $k$ that satisfies the condition of $2k\geq \max\{\deg(p),\max_i \deg(q_i)\}$.
The SDP relaxation of order $k$, which we denote $R_k$, has the form:
\begin{mini}
{ y=(y_{\omega})_{\lvert\omega\rvert\leq 2k} }{\sum_{\lvert\omega\rvert\leq d} p_{\omega} y_{\omega}}{R_k:}{p^k=}
\addConstraint{M_k(y)}{\succcurlyeq 0}{}
\addConstraint{M_{k-d_i}(q_iy)}{\succcurlyeq 0, }{i=1,\ldots,m}
\addConstraint{\langle\phi,\phi\rangle}{= 1.}{}
\end{mini}

Let us define the quadratic module, following \cite{Pironio2010}.  Let $Q=\{q_i,i=1,\dots,m\}$ be the set of polynomials determining the constraints. 
The \emph{positivity domain} $\mathbf{S}_Q$ of $Q$ are $n$-tuples of bounded operators $\mathbf{X}$ on a Hilbert space $\mathcal{H}$ making all $q_i(\mathbf{X})$ positive semidefinite.
The \emph{quadratic module} $\mathbf{M}_Q$ is the set of $\sum_if_i^{\dag}f_i+\sum_i\sum_j g_{ij}^{\dag}q_ig_{ij}$ 
where $f_i$ and $g_{ij}$ are polynomials in the $2n$ operators in $[\mathbf{X},\mathbf{X}^{\dag}]$. 
As in \cite{Pironio2010}, if the Archimedean assumption is satisfied, $\lim_{k \to \infty} p^k=p^*$ for a finite $k$.

\subsection{The Formal Statement}

In order to utilise subgroup fairness or instantaneous fairness, one needs to be able to guarantee that global optima for the corresponding optimisation problems can be found. Such guarantees are non-trivial in the case of non-convex, non-commutative optimisation problems.  
Following \cite{zhou2020proper}, one can formalise the guarantees we provide:

\begin{assume}[Archimedean]
\label{Archimedean}
Quadratic module $\mathbf{M}_Q$ of Eq.(5) or Eq.(6) is Archimedean, i.e., there exists a real constant $C$ such that $C^2-(X_1^{\dag}X_1+\cdots+X_{2n}^{\dag}X_{2n})\in \mathbf{M}_Q$,
for these $2n$ operators in $[\mathbf{X},\mathbf{X}^{\dag}]$. 
\end{assume}

Our main result shows that it is possible to recover the quadruple $(G,F,V,W)$ of the subgroup-blind $\mathcal{L}$ with guarantees of global convergence:

\begin{thm}
\label{T:covergence}
For any observable linear system $\mathcal{L}=(G,F,V,W)$, 
for any length $\mathcal{T}^+$ of a time window, 
and any error $\epsilon > 0$, 
under Assumption~\ref{Archimedean},
there is a convex optimisation problem from whose solution one can extract the best possible estimate of system matrices of a system $\mathcal{L}$ based on the observations, with fairness subgroup-fair considerations Eq.(5), up to an error of at most $\epsilon$ in Frobenius norm. Furthermore, with suitably modified assumptions, the result holds also for the instant-fair considerations Eq.(6).
\end{thm}

\begin{proof}
First, we need to show the existence of a sequence 
of convex optimisation problems, whose objective function approaches the optimum of the non-commutative polynomial optimisation problem.
As explained in the subsection above, 
\cite{Pironio2010} shows that, indeed, there are natural semidefinite programming problems, which satisfy this property.
In particular, the existence and convergence of the sequence is shown by Theorem 1 of \cite{Pironio2010}, which requires Assumption~\ref{Archimedean}.

Notice that we can use the so-called rank-loop condition of \cite{Pironio2010} to detect global optimality.
Once optimality is detected, it is possible to extract the global optimum $(\mathcal{H}^*, \mathbf{X}^*, \phi^*)$ from the optimal solution $y$ of problem $R_k$,
by Gram decomposition; cf. Theorem 2 in \cite{Pironio2010}.
Simpler procedures for the extraction have been considered, cf. \cite{henrion2005detecting}, but remain less well understood. 

More broadly, we would like to show the extraction of the minimiser from the SDP relaxation of order $k(\epsilon)$ in the series is possible. 
There, one utilises the Gelfand--Naimark--Segal (GNS) construction \cite{gelfand1943imbedding,segal1947irreducible}, as explained in Section 2.2 of  \cite{klep2018minimizer},
which does not require the  rank-loop condition to be satisfied,
We refer to Section 2.2 of  \cite{klep2018minimizer} and Section 2.6 of  \cite{dixmier1969algebres} for details.
\end{proof}

In summary, Theorem~\ref{T:covergence} makes it possible to recover the quadruple $(G,F,V,W)$ of the subgroup-blind $\mathcal{L}$ using the technologies of NCPOP with guarantees of global convergence \cite{Pironio2010}. 
To do so, we need to have the non-commutative versions of those formulations firstly, thus introduce a a Hilbert space $\mathcal{H}$ with all operators and a normalised vector defined on $\mathcal{H}$. 

\section{The COMPAS dataset}
\label{sec:COMPAS dataset}
COMPAS 
(Correctional Offender Management Profiling for Alternative Sanctions) is a popular commercial algorithm used by judges and parole officers for scoring criminal defendant’s likelihood of re-offending (recidivism). It has been shown that the algorithm is biased in favour of white defendants, based on a 2 year follow-up study (i.e., who actually committed crimes or violent crimes after 2 years).

Downloaded from \url{https://github.com/propublica/compas-analysis}, this dataset is what \cite{angwin2016machine} used in analysing the racial bias in COMPAS recidivism scores.
The COMPAS dataset comprises of defendants' gender, race, age, charge degree, COMPAS recidivism scores, two-year recidivism label, as well as information on prior incidents. 
The COMPAS recidivism scores, ranging from 1 to 10, are positively related to the estimated likelihood of recidivism, given by the COMPAS system.
The two-year recidivism label denotes whether a person actually got rearrested within two years (label 1) or not (label 0). If the two-year recidivism label is $1$, there is also information concerning the recharge degree and the number of days until the person gets rearrested.
The dataset also consists of information on
'Days before Re-offending', which is the date difference between the defendant's crime offend date and recharge offend date. It could be negatively correlated to the defendant's actual risk level while the COMPAS recidivism scores would be the estimated risk level.

\vskip 0.2in
\bibliography{ref}
\bibliographystyle{theapa}

\end{document}